\documentclass{article}

\PassOptionsToPackage{numbers, compress}{natbib}
 \usepackage[preprint]{neurips_2026}

\usepackage[utf8]{inputenc} 
\usepackage[T1]{fontenc}    
\usepackage{hyperref}       
\usepackage{url}            
\usepackage{booktabs}       
\usepackage{amsfonts}       
\usepackage{nicefrac}       
\usepackage{microtype}      
\usepackage{xcolor}         

\usepackage{amsmath}
\usepackage{amssymb}
\usepackage{mathtools}
\usepackage{amsthm}
\theoremstyle{plain}
\newtheorem{theorem}{Theorem}[section]
\newtheorem{lemma}[theorem]{Lemma}
\newtheorem{definition}[theorem]{Definition}

\newtheorem{conjecture}[theorem]{Conjecture}

\usepackage{algorithm}
\usepackage{algpseudocode}

\newcommand{\w}{\vec{w}^{(1)}_{\alpha}}
\newcommand{\ww}{\vec{w}^{(2)}_{\alpha}}

\usepackage{subcaption}
\usepackage{multirow}

\title{It's Not a Lottery, It's a Race: \\ Understanding How Gradient Descent Adapts the Network's Capacity to the Task}

\author{Hannah Pinson \\
  Data and AI cluster, EAISI \\
  Eindhoven University of Technology\\
  \texttt{h.pinson@tue.nl} \\
}

\begin{document}

\maketitle
\begin{abstract}
  Our theoretical understanding of neural networks is lagging behind their empirical success. One of the important unexplained phenomena is why and how, during the process of training with gradient descent, the theoretical capacity of neural networks is reduced to an effective capacity that fits the task. We here investigate the mechanism by which gradient descent achieves this through analyzing the learning dynamics at the level of individual neurons in single hidden layer ReLU networks. We identify three dynamical principles, namely mutual alignment, unlocking and racing, that together explain why we can often successfully reduce capacity after training through the merging of equivalent neurons or the pruning of low norm weights. We specifically explain the mechanism behind the lottery ticket conjecture, or why the specific, beneficial initial conditions of some neurons lead them to obtain higher weight norms. 
\end{abstract}

\begin{figure}[h!]
    \centering
\includegraphics[width=0.80\linewidth]{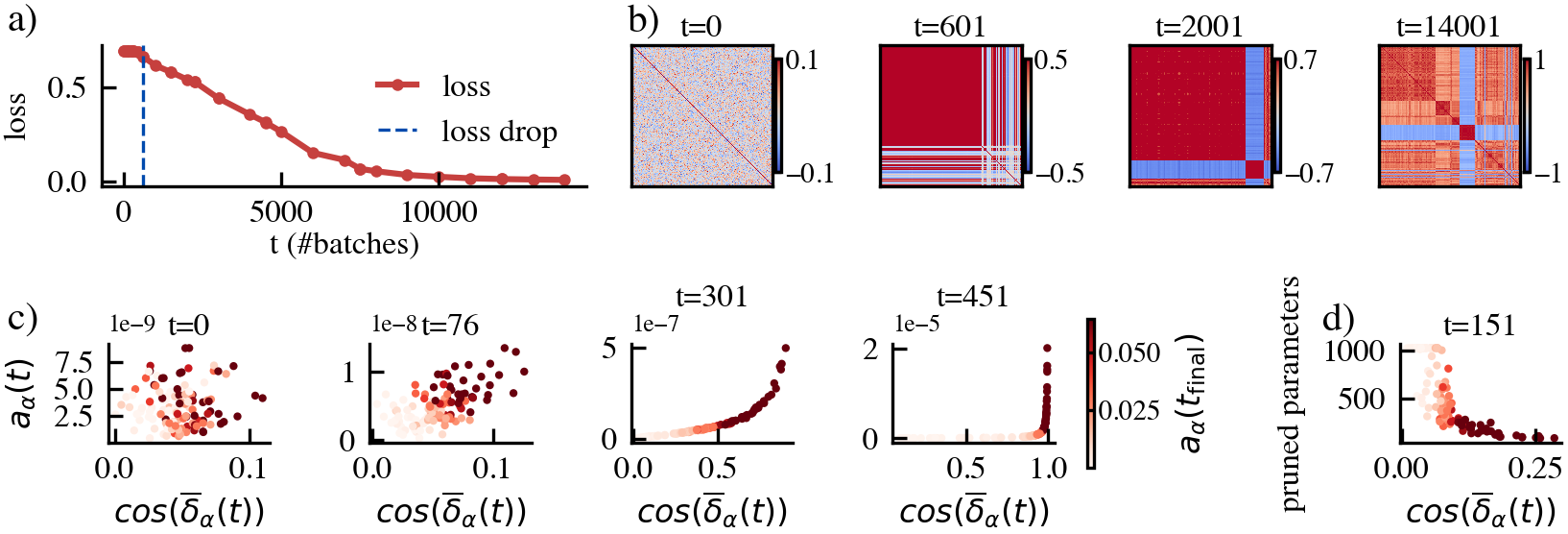}
    \caption{Overview of results for a network of d=300 neurons on a CIFAR-10 based dataset of N=200 samples, trained from small initial conditions. a) Training loss. The loss drops around t=300. b) Cosine similarity between the total weight vectors of each neuron, organized by detected clusters (cos\_sim > 0.99). \emph{Groups of neurons thus mutually align over time. These groups can be merged.} c) Norm product $a_{\alpha}(t)$ of each neuron $\alpha$ of the largest cluster, in function of the alignent $\overline{\delta_{\alpha}}(t)$ at different timepoints very early in training. Colors correspond to final norm. \emph{Thus the closer the alignment in the early phase of learning, the higher the obtained norm at the end of training (within a group of neurons)}. d) Number of individual weights pruned per neuron in function of $\overline{\delta_{\alpha}}(t)$ computed at a very early point in training, for the largest group of neurons. Individual weights are pruned using IMP, the algorithm to find lottery tickets \cite{frankle_lottery_2019}, with a target sparsity of 50\% over 3 rounds. Colors correspond to final norm. \emph{The alignment early in training can thus be used to predict the neurons for which a high number of parameters would be pruned to obtain lottery tickets.} We here derive a theoretical framework to explain these dynamics and the nature of alignment $cos(\overline{\delta_{\alpha}}(t))$. }
    \label{fig:main_figure}
\end{figure}

\section{Introduction}
Modern neural networks are very large or even overparameterized: they have more parameters than input samples. After training, they can often be pruned or compressed \cite{zhu_prune_2017, FrantarA23, hoeer_sparsity_2021, 10643325}, indicating that the final function they implement can actually be represented using (far) fewer parameters.  This excess of parameters leads to higher computational costs. Still, it is found that overparameterized networks generalize better, a phenomenon explained by the implicit bias of gradient descent  \cite{ neyshabur_search_2015, du_power_2018, nacson_convergence_2019, ji_directional_2020, advani_high-dimensional_2020}. A closely related concept is that of lottery tickets (for a survey, see \cite{liu_survey_2024}). In the original paper on lottery tickets \cite{frankle_lottery_2019}, the authors show experimentally that, within a randomly initialized overparameterized network, there exists a (much) smaller subnetwork that can reach at least a similar performance as the full network. They showed such networks can be discovered based on the pruning of weights with lower norms. The exact initialization values are key: the discovered subnetwork needs to be trained from its original initial conditions. The name for these types of subnetworks, ‘lottery tickets', reflects their lucky initialization. From this point of view, training overparameterized networks with gradient descent can be seen as a mechanism to select optimally initialized subnetworks through the amplification of their parameter norms \cite{morcos_one_2019}. The existence of the subnetworks at initialization is the so-called \emph{lottery ticket hypothesis}. The idea that gradient descent selects these subnetworks is denoted \emph{lottery ticket conjecture} \cite{frankle_lottery_2019}. \\ 
It has been shown that lottery tickets can be found in a wide range of architectures \cite{liu_survey_2024}, including transformers and LLMs \cite{brix_successfully_2020, behnke_losing_2020}, and there is a proof of the lottery ticket hypothesis itself \cite{malach2020proving}. In terms of the conjecture, i.e., how gradient descent amplifies the norms of the lottery tickets, there are some limited results: lottery tickets can be discovered through pruning already in the early stages of training \cite{frankle_early_2020}, the signs of the initializations matter more than the values \cite{zhou2019deconstructing}, and the same lottery tickets can be used for similar datasets \cite{morcos_one_2019}, indicating a strong dependence of a lottery ticket on the (general) features of the dataset it was trained on. In 
\cite{paul2022unmasking, frankle_linear_2020}, the authors link the success of (iterative) pruning used to discover lottery tickets to the robustness of gradient descent and linear mode connectivity. However, both the exact nature of the
beneficial initial conditions, as well as how exactly gradient descent leads to lottery ticket obtaining a larger norm compared to
other parts of the network, remains an open question.\textbf{Our contributions are the following}:
\begin{itemize}
\itemsep-0em 
    \item we develop a novel framework to analyze the early learning dynamics, i.e., the evolution of the parameters under the influence of gradient descent, of individual neurons in single hidden layer \emph{ReLU} networks for binary classification tasks, and this in the feature learning regime \cite{woodworth_kernel_2020}, and we identify the three important dynamical principles of mutual alignment, unlocking and racing, 
    \item we explain how these dynamical principles in our specific setting can give rise to a reduction in effective capacity in terms of equivalent neurons that can be merged,
    \item we reveal how these principles lead to corresponding neurons obtaining a higher norm for their associated parameters, and explain the nature of the beneficial initial conditions of lottery tickets.
\end{itemize}
The notion of (mutual) alignment we here discuss is not novel: it has been shown before that, under the influence of gradient descent, weight vectors of networks trained in the rich regime tend to align to a (smaller) set of target directions related to the dataset \cite{soudry_implicit_2018, gur-ari_gradient_2018, atanasov_neural_2022}.  These directions can be exactly or approximately calculated for simplified settings, such as linear networks or single neurons \cite{atanasov_neural_2022, saxe_neural_2022, refinetti_neural_2023}.  Alignment is directly related to a  reduction in effective capacity \citep{Baratin21, kopitkov_neural_2020, paccolat_geometric_2021}, as it reduces the dimension of the feature space implemented by the weights.  In \cite{atanasov_neural_2022}, the authors introduce the notion of ‘silent aligment': the fact that this alignment often takes place early in training, \emph{before} the loss considerably drops, yielding a clear separation in time between changes in weight vector directions and changes in their norm (and subsequently in the loss). The crucial and new observation we here make is that there is no complete decoupling between changes in direction and changes in norm: we find, in fact, that the growth in norm exponentially depends on the angular distance to the target direction. This creates a ‘winner-takes-all' dynamic \cite{fukai1997simple}: neurons which are initialized closer to their target directions can grow more quickly in norm. Optimally aligned neurons that increase in norm, reduce the loss and solve the task. Thus the neurons that arrive earlier strongly reduce the loss and gradients, and therefore they inhibit the development of other neurons. In general, a few neurons which have the ‘luck' to be initialized close to their target direction, thus obtain a higher norm and contribute more to the solution of the task; at the same time, other, less important neurons do not grow significantly in norm. The latter neurons can subsequently be pruned. In short, we thus find that gradient descent implicitly selects lottery tickets in a process more akin to dynamical race than a static lottery. \\
\textbf{Related Work} Our analysis is closely related to the study of gradient descent dynamics for linear networks \cite{saxe_mathematical_2019, tarmoun_understanding_2021, j_domine_exact_2023, pinson_linear_2023, gidel_implicit_2019} and extensions to nonlinear networks in simplified settings \cite{tachet_learning_2020, refinetti_neural_2023, jarvis_make_2025}.  We here use a ReLU network including a softmax function and a cross-entropy loss, which makes our setting closer to networks used in practice, and these works assume completely silent alignment, and/or do not focus on the dynamics of individual neurons within a network. 
Another line of related work relates iterative pruning based on magnitudes to the discovery of inductive biases \cite{pellegrini2022neural, redman_how_2025}. Lastly, some works study the prediction of lottery tickets at initialization or early in training. In \cite{frankle_linear_2020, frankle_early_2020}, the authors show experimentally that lottery tickets are largely determined early in training, and multiple experimental methods exist to perform pruning at those early stages \cite{lee_snip_2019, wang_picking_2020}. Our work provides the theoretical basis to understand the success of those methods.

\section{Preliminaries and Notation}

\textbf{Single Hidden Layer ReLU Networks} Let $\{(\vec{x}^s, \vec{y}^s)\}_{s=1}^N \subseteq \mathbb{R}^N \times \{ \left[ \begin{smallmatrix} 1 \\ 0 \end{smallmatrix} \right], \left[ \begin{smallmatrix} 0 \\ 1 \end{smallmatrix} \right] \} $ be a binary classification dataset. We consider the single hidden layer neural network with $d$ neurons in the hidden layer:
$\vec{h}^{s} =  ReLU( W^{(1)}  \vec{x}^s + \vec{b}^{(1)} )  \; ; \vec{z}^{s} =  W^{(2)}  \vec{h}^{s}    + \vec{b}^{(2)} \; ;  \hat{\vec{y}}^s = softmax(  \vec{z}^{s}  )
$,
trained with gradient descent for a classification task using a binary cross-entropy loss, in our setup with two output classes: $L^s = - \sum_{i=0}^{1} \vec{y}_i^s \log(\hat{\vec{y}}_i^s) $ for a single sample s. 
We denote the vector of incoming weights to a neuron indexed by $\alpha$ as $\w =  W^{(1)}_{\alpha,:}$ and of its outgoing weight vector as $\ww = \mathbf{W}^{(2)}_{:, \alpha}$. The $d_{input}-1$ angles of $\w$ in a (hyper)spherical coordinate system are denoted $\theta_{\w, i}$ with $i \in {0,...,d_{input}-1}$ and $d_{input}$ the dimension of the input. The angle of $\ww$ is denoted $\theta_{\ww}$. We furthermore define $\vec{\Delta y^s} = \vec{y^s} - \vec{\hat{y}}^s$, with corresponding angle $ \theta_{\vec{\Delta y^s}} $. \\
\textbf{ReLUs as gates and effective datasets} ReLU activation functions effectively act as gates: positive values are passed without modification, and negative values result in zero. At any point during training, we can determine what we call the ‘gating vector' $\vec{g}_{\alpha}$ of a neuron $\alpha$:  $\vec{g}_{\alpha}$ is a vector of dimension $d_{input} \times 1$ with values either 0 or 1, such that $g_{\alpha}^s$ represents whether the incoming sample is passed on ($g_{\alpha}^s=1$) or blocked ($g_{\alpha}^s=0$) by the ReLU activation of neuron $\alpha$. From this point of view, each neurons only ‘sees' an effective dataset given by the samples that pass its gates, as the other samples do not contribute to its activations or gradients. The total task/dataset can thus be seen as an overlapping collection of effective tasks/datasets, and if we only consider its effective samples, each neuron acts as a linear function. In general $\vec{g}_{\alpha}$ changes over training time, i.e., $\vec{g}_{\alpha} = \vec{g}_{\alpha}(t)$. 
In terms of notation, we will use $\langle \rangle$ to denote an average over all samples that pass the ReLU gate of the given neuron at the given point in training, i.e., the average over the effective dataset of a given parameter $p_\alpha$: $
    \langle p_{\alpha} \rangle = \frac{1}{N}\sum_s^{N} g^s_\alpha p_{\alpha}^s = \frac{1}{N}\sum_s^{N^{e}} p_{\alpha}^s
$
, with $N^{e}$ the number of current effective samples.  We will use the notation $\langle p_{\alpha} \rangle_{C_1}$ to denote an average over all current \emph{effective} samples of class 1, and similar for class 2. $N^{e}_1$ and $N^{e}_2$ denote the number of effective samples of class 1 and 2. \\
\textbf{Gradient flow} When the learning rate is low, the parameters of the network change only slightly with each new sample provided to the network. Under these conditions, we can approximate the change in a parameter over a set of samples through using the average of the gradients obtained over this set of samples $ \Delta p  \approx - \lambda N \langle \frac{\partial L^s}{ \partial p} \rangle$, with $\lambda$ the learning rate and N the number of samples/updates. We can then consider the continuous time limit, and obtain for each parameter the differential equation $\frac{\partial p}{\partial t}  = - \lambda  \langle \frac{\partial L^s}{ \partial p} \rangle
    \label{eq:gradient_flow}$. (see also \cite{saxe_mathematical_2019}.)

\section{Mathematical Theory}

The learning dynamics in general follow the complicated, coupled system of nonlinear equations given by the equations of gradient descent. However, we will show that we can consider different phases in training during which the dynamics can be drastically simplified, and from the simplified equations during these phases we can then explain how the phenomena of alignment, unlocking and racing arise. The assumptions we will keep throughout the theoretical sections are: the architecture of the network and loss (see above), the initial conditions (small random initial values drawn from a given gaussian distribution, zero biases), and a small learning rate such that we can consider gradient flow. 

\subsection{Equations of Gradient Descent}
We first define the vector $\langle \vec{\gamma}_{\alpha}^s \rangle$, which, as we will show, plays a central role in the dynamics:
{\small
\begin{align}
   & \langle \vec{\gamma}_{\alpha}^s \rangle = \frac{1}{N} \sum^N_s \vec{g}^{s}_\alpha  \cos(\theta_{\vec{\Delta y}^{s}} - \theta_{\ww}) \, ||  \Delta \vec{y}^s  || \,  \vec{x}^{s} 
   \label{eq:gamma}
\end{align}}
which depends on time through $\vec{g}^{s}_\alpha = \vec{g}_{\alpha}^s (t)$ and $ \Delta \vec{y}^s =  \Delta \vec{y}^s(t)$. We can then write out the gradients in full \emph{for a given gating vector $\vec{g_\alpha}$} in a form that will allow us to more easily derive the dynamical principles we are interested in (details appendix \ref{app:eq_grad_descent}).  The notation $\langle \rangle$ thus denotes an average over the corresponding effective dataset:
{\small
\begin{align}
     & \!\langle \frac{\partial L^s}{ \partial \theta_{\ww}} \rangle =- ||\vec{w}^{(2)} _{\alpha}|| \; \langle h_\alpha^s \;  ||\vec{\Delta y}^{s}|| \;  
      \sin{(\theta_{\vec{\Delta y^{s}}}- \theta_{\ww}}) \rangle  \label{eq:grad_phi_alpha} \\ 
     & \! \langle \frac{\partial L^s}{ \partial ||\vec{w}^{(2)}_{\alpha}||} \rangle =  - \langle h_\alpha^s \;  ||\vec{\Delta y}^{s}|| \;  \cos{(\theta_{\vec{\Delta y^{s}}}- \theta_{\ww}}) \rangle \label{eq:grad_norm_ww_alpha} \\
      & \! \langle \frac{\partial L^s}{ \partial \theta_{\w,i}} \rangle \! = \! - \!||\vec{w}^{(2)}_{\alpha}|| \; ||\vec{w}^{(1)}_{\alpha} || \; ||\langle\vec{\gamma}_{\alpha}^s \rangle  || \; \sin{(\theta_{\langle\vec{\gamma}_{\alpha}^s \rangle,i } - \theta_{\w,i})} \label{eq:grad_theta_alpha} \\
      & \! \langle \frac{\partial L^s}{ \partial ||\vec{w}^{(1)}_{\alpha}||} \rangle  = - ||\vec{w}^{(2)}_{\alpha}|| \; ||\langle\vec{\gamma}_{\alpha}^s \rangle  || \;\cos{(\theta_{\langle\vec{\gamma}_{\alpha}^s \rangle,i } - \theta_{\w,i})}\\
       & \langle\frac{\partial L^s}{\partial b_\alpha}\rangle = - || \vec{w}^{(2)}_{\alpha} || \; \langle ||\vec{\Delta y^s}|| \;  \cos{(\theta_{\vec{\Delta y}^s} - \theta_{\ww} )} \rangle
      \label{eq:grad_norm_w_alpha}
    \end{align} }

\subsection{Phase 1: Mutual Alignment}

We here show that in the very early phase of training and when starting from small initial conditions, the weight vectors of the neurons form subgroups that each converge in direction to a set of dataset-dependent target directions. These directions correspond to naive classifiers of specific effective datasets. We do this by first defining a set of assumptions that simplify the dynamics, and we then show how the neurons behave under these simplified settings. Finally, we argue why the assumptions are valid in the early phase of learning when starting from small initial conditions.




\begin{figure}[h!]
    \centering
   \includegraphics[width=\linewidth]{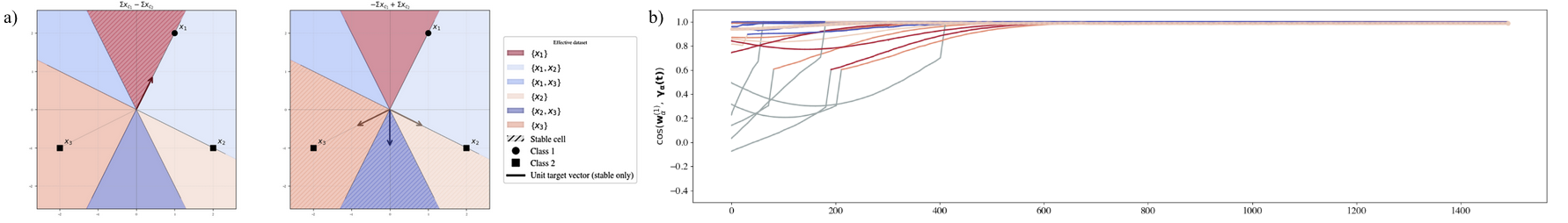}
    \caption{  Illustration with a toy dataset of 3 samples. (a) resulting cells with indications of the corresponding effective dataset, stable cells, and target vectors (only shown for the stable cells). (b) the randomly initialized $\w$ follow a trajectory where they each time converge in direction to a target direction that can lie outside (unstable cells, gray part of the lines) or inside (stable cells, colors according to panel (a)) their current cell. After some boundary crossings, they eventually reach their stable cell.}
    \label{fig:placeholder}
\end{figure}


To handle the ReLU activation functions we first consider the following lemma:

\begin{lemma} \textbf{Number of cells}
    Consider $N$ data samples $\vec{x}_i \in R^{d_{input}}$. When all biases $b_\alpha =0$, the space can be divided in K conic regions $\mathcal{C}_k$ of constant effective dataset and gating vector:  for any two weight vectors $\w ,\vec{w}^{(1)}_{\beta} \in \mathcal{C}_k$ we have $\vec{g}_{\alpha} = \vec{g}_{\beta} $. The number of such regions $ K  $ is upper bounded by $ 2 \sum_{i=0}^{d_{input}-1} \binom{N-1}{i} $. We call these regions cells.  \\\textbf{Proof:} When all $b_\alpha =0$, the effective dataset for a neuron $\alpha$ consists of the samples $x_i$ for which $\w \cdot x_i > 0$ (as these samples will pass the ReLU function). Regions of constant effective datasets are bounded by the hyperplanes through the origin orthogonal to each sample  vector $\vec{x}_i$, given by $\w \cdot \vec{x}_i = 0$. As we have N such planes, the number of regions  K is upper bounded by $ 2 \sum_{i=0}^{d_{input}-1} \binom{N-1}{i} $ (see \cite{cover2006geometrical}). 
\end{lemma} 

 As the effective datasets are by definition constant within each cell $\mathcal{C}_k$ , each cell $\mathcal{C}_k$ has an associated fixed gating vector we denote $\vec{g}_k$. A weight vector $\w$ is randomly initialized in one of the cells $\mathcal{C}_k$. As the weights are updated, $\w$ changes direction. While $\w$ remains in the same cell $\mathcal{C}_k$ , its gradient updates (eqs. \ref{eq:grad_phi_alpha} - \ref{eq:grad_norm_w_alpha}) are based on the gating vector $g_{k}$. If  $\w$ crosses the boundary from one cell to the next, its gating vector $\vec{g}_\alpha(t)$ changes and its gradient updates are now based on the new gating vector $g_{k'}$ associated to the new cell $\mathcal{C}_{k'}$.

\begin{definition}\textbf{Simplified dynamics phase 1}  For a neuron $\alpha$, we make the following assumptions whenever we assume 'simplified dynamics' in this phase: $b_\alpha =0$, $||\w||$ constant and small, $||\ww||$ constant and small, and $\theta_{\ww} = \frac{7\pi}{4} + 2\pi$ ("aligned to predict class 1") or  $\theta_{\ww} = \frac{3\pi}{4} + 2\pi$ ("aligned to predict class 2").
\end{definition}

\begin{definition}\textbf{Target vectors associated to a cell}
   We associate to each cell $\mathcal{C}_k$ and corresponding gating vector $\vec{g}_k$ two vectors: $\vec{\gamma}^{*,simpl.}_{1,k} = \frac{\sqrt{2}}{2} (
      N^{e}_1 \langle  \, \vec{x}^s  \rangle_{C_1} - N^{e}_2 \langle \, \vec{x}^s  \rangle_{C_2} )$ and $\vec{\gamma}^{*,simpl.}_{2,k} = \frac{\sqrt{2}}{2} (N^{e}_2 \langle  \, \vec{x}^s  \rangle_{C_2} - N^{e}_1 \langle \, \vec{x}^s  \rangle_{C_1})$ \textbf{Note}: for the next order approximation, see appendix \ref{app:apporx_target_directions}.
\end{definition}
\vspace{-5pt}
These vectors $\vec{\gamma}^{*,simpl.}_{c,k}, c \in \{1,2\} $ correspond to a naive classifier \cite{refinetti_neural_2023}: the direction that separates the samples in the given effective dataset based solely on the mean of the classes, but with a different sign. We now connect the direction of the weights $\w$ in a cell $\mathcal{C}_k$ to the target directions $ \vec{\gamma}^{*,simpl.}_{c, k}, c \in \{1,2\}$ associated to the cell. To ease notation, we define $c(\alpha)$ as  $c(\alpha) = 1$ when $\theta_{\ww} = \frac{7\pi}{4} + 2\pi$ and $c(\alpha) = 2$ when $\theta_{\ww} = \frac{3\pi}{4} + 2\pi$.
\vspace{-5pt}
\begin{lemma} Consider $\w \in \mathcal{C}_k$ and simplified dynamics. As long as $\w \in \mathcal{C}_k$, $ \langle\vec{\gamma}_{\alpha}^s \rangle (t)$  is  a constant vector given by $\vec{\gamma}^{*,init}_{c(\alpha),k}$ and $\theta_{\w,i}(t) \rightarrow \theta_{\vec{\gamma}^{*,init}_{c(\alpha),k},i } $.  \textbf{Proof outline}: In the case of simplified dynamics,  eq. \ref{eq:grad_theta_alpha} reduces to:
$\langle \frac{\partial L^s}{ \partial (\theta_{\w,i})} \rangle \! =  - C_{\alpha]} || \vec{\gamma}^{*,simpl.}_{c(\alpha),k} || \; \sin{(\theta_{\vec{\gamma}^{*,simpl.}_{\alpha),k},i} - \theta_{\w,i})}  
$ with $C_\alpha$ a constant. Note that this equation is in this case decoupled from the other equations. We then compute the minima and second derivatives to determine the stability. Detailed proof, see appendix \ref{app:proof_lemma_conv}. 
\end{lemma}
\vspace{-5pt}
While within the same cell $C_k$ and under simplified dynamics, the direction of $\w$ will thus be updated towards the constant target direction $\vec{\gamma}^{*,simpl.}_{c(\alpha),k}$. This target direction can point inward or outwards the current cell, meaning the weight vector will either remain in the current cell (at least as long as the simplified dynamics hold) or will cross to a different cell. We can thus consider stable and unstable cells, and arrive at our final theorem: 
\begin{definition}
A cell $\mathcal{C}_k$ is stable for a neuron $\alpha$ if $\vec{\gamma}^{*,simpl.}_{c(\alpha), k} \in \mathcal{C}_k$ and unstable if $\vec{\gamma}^{*,simpl.}_{c(\alpha),k} \notin C_k$. 
\end{definition} 

\begin{theorem} \textbf{Groups of neurons converge to the same direction}
    Under simplified dynamics, the randomly initialized $\w$ will start out in a stable cell or they will follow a trajectory from (unstable) cell to (unstable) cell, until they enter a stable cell. When $d > K_{stable}$, where d is the number of neurons in the hidden later, subgroups of neurons will eventually enter the same stable cell. They remain in this cell and converge in direction to the target angle:  $\forall \w \in \mathcal{C}_k, i: \theta_{\w,i}(t) \rightarrow \theta_{\vec{\gamma}^{*,simpl.}_{c(\alpha),k},i}$. The target angle $\theta_{\vec{\gamma}^{*,simpl.}_{c(\alpha),k},i}$ is the same for all neurons $\alpha$ that have the same alignment of their outgoing weights (c($\alpha$)= 1 or c($\alpha$)=2).  \textbf{Proof:} this directly follows from the lemmas above and the pigeonhole principle. 
\end{theorem}

\begin{conjecture}
For a structured dataset, the number of stable cells $K_{stable}$ is much smaller than the total number of cells $K$ (see also \cite{hanin2019complexity}).
\end{conjecture}
\begin{conjecture}\label{conj_d}
Even in the case $d<K_{stable}$, multiple $\w$ can converge to the same stable cell, depending on the structure of the dataset. 
\end{conjecture}
In appendix \ref{app:stable_cells}, we propose an algorithm to efficiently explore the space and estimate the number of stable cells. For the CIFAR-10 based datasets we consider, we consistently find an estimate for the number of stable cells in the range $K_{stable}=10^1 - 10^2$, a tiny fraction compared to the total number of cells (see  appendix \ref{app:stable_cells}) . Then, for Conj. \ref{conj_d}, consider the basin of attraction for each stable cell: the region of space over which, if a weight vector is initialized in this space, it will converge to the stable cell. We hypothesize that structure in the data, e.g., dense regions versus sparse regions, results in different sizes of the basins of attraction. Neurons are thus not uniformly divided over stable cells, and some stable cells could attract a lot of neurons while others might have no neurons initialized in their (small) basin of attraction. \\
We will now argue why the initial dynamics of the network adhere to the simplified dynamics. First note that from the general dynamics we observe:

\begin{lemma}\label{lemma:decoupling}\textbf{Initial Direction-Norm Decoupling (cfr. \cite{atanasov_neural_2022})}
The ratio of gradient updates is proportional to $  \frac{\partial L^s}{ \partial (\theta_{\ww})}  /  \frac{\partial L^s}{ \partial ||\vec{w}^{(2)}_{\alpha}||} \sim  \tan( \theta_{\vec{\Delta y^{s}}}- \theta_{\ww} )$
A similar observation holds for $\theta_{\w,i}$.  \textbf{Proof:} This can be directly observed from the gradient equations, when rewritten as eqs. \ref{eq:grad_phi_alpha} - \ref{eq:grad_norm_w_alpha}).
\end{lemma}
This means that, in the early phase of learning, the weight vectors $\w$ and $\ww$ that are far (in angular distance) to their target angles, will receive updates on their angles that are orders of magnitude higher than updates on their norms (illustrated in appendix fig. \ref{app:decoupling} ). As we discussed before, this initial decoupling of norm updates from angle updates has been called ‘silent alignment' \cite{atanasov_neural_2022}: starting from small initial conditions, vectors generally align to their initial target direction \emph{before} their norms start to increase and the loss starts to drop. The outgoing weights $\ww$ are 2D vectors, and crucially, they have fixed targets as $\theta_{\vec{\Delta y^{s}}} = \frac{7\pi}{4} \vee \frac{3\pi}{4} $ at all times. A detailed analysis can be found in appendix \ref{app:ww}. As the vectors are randomly initialized in a 2D space, they start out relatively closer to the target vectors (as compared to vectors initialized in a high dimensional space) and since they have fixed target vectors, we can assume they align faster than the incoming weights $\w$. We can divide the incoming weights in two categories: neurons that go through fast, silent alignment (phase 1), because they are initialized outside their stable cells, and neurons that are initialized in their stable cell. The latter group immediately starts out in phase 2 (see next).  Together, this justifies the assumptions underlying the simplified dynamics for the considered neurons in this phase: $b_\alpha =0$, $||\w||$ constant and small due to silent alignment of $w$; the growth of   $||\ww||$ depends on the growth of  $||\ww||$ and therefore also remains small, and $\ww$ aligns faster than $\w$. 

\subsection{Phase 2: Unlocking}


Now we consider the next phase: a group $\w$ has entered their shared stable cell but are not fully aligned to their target direction. We will show that they obtain an exponential difference in their norms, based on their relative angular distance to the shared target direction. Let $\delta_{\alpha,i}$ be the remaining angular distance $\delta_{\alpha,i} = |\theta_{\vec{\gamma}^{*,init}_{c(\alpha),k},i } - \theta_{\w,i}|$. 

\begin{definition}\textbf{Simplified dynamics phase 2} 
For a neuron $\alpha$, we make the following assumptions whenever we assume 'simplified dynamics' in this phase: $b_\alpha =0$, $\theta_{\w}$ constant, $||\w||$ and $||\ww||$ both small (but not constant), and $\theta_{\ww} = \frac{7\pi}{4} + 2\pi$  or  $\theta_{\ww} = \frac{3\pi}{4} + 2\pi$.
\end{definition}

The assumptions of the simplified dynamics for this phase are valid because, in this phase, we have the inverse effect of the direction-length decoupling (see lemma \ref{lemma:decoupling}): directions are updated much slower as compared to norms. Apart from that, the bias starts out orders of magnitude smaller than the weights (as $b_\alpha = 0$ initially), hasn't significantly grown before this phase, and thus still has little influence $b_\alpha \approx 0$.  Under these simplified dynamics, we find:

\begin{theorem}
    Consider the product $a_{\alpha} =  ||  \vec{w}_{\alpha}^{(2)} || \; ||  \vec{w}_{\alpha}^{(1)} ||$. Under the simplified dynamics for this phase, we have:
    \vspace{-3pt}
    \begin{equation}
         a_{\alpha}(t) \sim  e^{ \lambda \; || \vec{\gamma}^{*,init}_{c(\alpha),k}|| \;  cos(\delta_\alpha) \; t}
         \label{eq:sol_a_exp}
    \end{equation}\vspace{-3pt}
  \textbf{Proof outline}: We have $
    \partial L / \partial ||  \vec{w}_{\alpha}^{(1)} || = - ||w^{(2)}_{\alpha}||  || \vec{\gamma}^{*,init}_{c(\alpha),k}||  \cos(\delta_\alpha) ,  \; 
     \partial L  / \partial ||  \vec{w}_{\alpha}^{(2)} ||   = -  ||w^{(1)}_{\alpha}|| \;  || \vec{\gamma}^{*,init}_{c(\alpha),k} || \;cos(\delta_\alpha) $, a coupled set of equations. Under the simplified dynamics for this phase, we can  then derive $
     \frac{\partial a_\alpha}{ \partial t } = \lambda  || \vec{\gamma}^{*,init}_{c(\alpha),k}|| cos(\delta_\alpha) a_{\alpha}\label{eq:evo_a}$  (detailed proof see appendix \ref{app:unlocking}).
\end{theorem}

Consider two neurons, $\alpha$ (well-aligned) and $\beta$ (less aligned), i.e.,  $\delta_\alpha > \delta_\beta$. During the unlocking phase, the ratio of their norms evolves as $\frac{a_{\alpha}(t)}{a_{\beta}(t)} \sim  e^{\lambda \|\gamma\| (\cos \delta_{\alpha} - \cos \delta_{\beta}) t}$. This indicates that even a marginal advantage in alignment is amplified exponentially. While we cannot exactly describe the dynamics beyond this phase, we find experimentally that this exponential difference in loss obtained early in training \emph{largely persists} to the end of training. We conjecture that this is the case exactly because it is exponential in nature, and because the racing nature of the dynamics after this point reduce the ability of other neurons to catch up (see next section). \\
\textbf{Note} As the norms are starting to grow in this phase, this phase generally corresponds to the period of training right before and after the loss drops. In fact, it's more precise to consider the drop in the loss over the \emph{effective} samples of this group, not over the total dataset (see note on asynchronous learning speed in the next section).


\subsection{Phase 3: Racing}

As the magnitudes of the weight vectors start to increase, the loss starts to considerably drop and $\vec{\Delta y}^s$ and $\langle \vec{\gamma}_{\alpha}^s \rangle$ are no longer constant, the biases grow and change the nature of the gating vectors, and all equations are truly coupled. We can no longer describe the dynamics using simplifications, but we conjecture there is a race ("winner-takes-all") dynamics between different neurons of the same group: 

\begin{conjecture}
   For neurons in the same group indexed by $\beta$, the dynamics of their norms after the unlocking phase can be described by a set of equations:
   \begin{equation}
   \frac{\partial a_\beta}{ \partial t } \approx \lambda \;  || \langle \vec{\gamma}^s \rangle (t) || \; cos(\delta_\beta) \; a_{\beta},
   \end{equation}
   one equation for each neuron $\beta$, with the value $ || \langle \vec{\gamma}^s \rangle (t) ||$ shared between all neurons of the group.  The value $ || \langle \vec{\gamma}^s \rangle (t) ||$ decreases as the error on the effective dataset decreases (see eq. \ref{eq:gamma}). If one neuron decreases the error through an increase in its norm, it exerts a lateral inhibition on the growth of the other neurons in the group through a decrease in the shared value $ || \langle \vec{\gamma}^s \rangle (t) ||$ for all neurons. 
\end{conjecture}

Thus, during the unlocking phase, the neurons that are closest to the target directions obtain the highest norms, and therefore contribute most to the decrease in error vectors and loss. The decreased loss subsequently slows down the learning for \emph{all} neurons of the group. This can be seen as a kind of race between neurons: the neurons that arrive first near the target direction exponentially grow in norm compared to those further away, and those early arriving neurons are the ones mainly solving the effective task (i.e., the task defined by the effective dataset) through reducing the loss. In other words, they win the race to become the important neurons in solving this specific (sub)task.

\section{Relationship to Effective Capacity}

We consider two ways in which the theoretical capacity, given by the number of neurons $d$, can be reduced to an effective capacity:  neurons can "play the same role" and could therefore be merged to a single effective neuron without changing the computations performed by the network. The second reduction we consider is the case where neurons are "not playing a meaningful role", i.e., do not contribute to the computations of the network, and could effectively be pruned. The hallmark of this is a low norm for the neuron's associated weight vectors as compared to other neurons in the network. \\
Our theoretical framework explains how both phenomena arise from training with gradient descent. Two neurons $\alpha$, $\beta$ that both very closely align to a mutual (initial) target direction can be merged to an effective neuron with effective norm effective norm: $||\mathbf{w}^{tot}_{effective}|| = \sqrt{ ||\mathbf{w}^{tot}_{\alpha}||^2 + ||\mathbf{w}^{tot}_{\beta}|| ^2 }$. On the other hand, between two neurons $\alpha, \beta$ with the same initial target direction, if neuron $\alpha$ arrives (much) earlier and thus aligns more closely than neuron $\beta$, its norm will grow much more than the norm of neuron $\beta$, and it will contribute much more to solving the task associated to their effective dataset. Neuron $\beta$, with a lower final norm,  will therefore thus be a better candidate for norm based pruning. Some important considerations:\\
\textbf{Note 1: Lottery Tickets} The IMP algorithm (iterative magnitude based pruning \cite{frankle_lottery_2019}) used to find lottery tickets, differs from our framework in some key ways: pruning of low norm weights happens on individual weights, over the whole network, and in rounds of pruning and retraining. Our framework concerns the differences in the norms of neurons, not individual weights, and this \emph{within} but not \emph{between} groups of neurons. In the experimental section, we explore how our framework relates to the results obtained with the IMP algorithm. \\
\textbf{Note 2: Predictions at Initialization and Early in Training}  The beneficial initial conditions that would make a neuron a ‘winning ticket' are then these initializations that lead to a smaller $\delta_{\alpha,i}$ compared to the other neurons in the group. In principle, there is no way to exactly predict this at initialization, as we would need to know the evolution of the neurons and their gating vectors in the first steps of training to exactly predict each $\delta_{\alpha,i}$. However, we find experimentally that the angular distance computed to target vectors \emph{computed at initialization and early points in training} is already predictive of the final relative norm of each neuron (see experiments).  \\
\textbf{Note 3: Asynchronous Learning Speeds} the different groups evolve at different speeds, because the effective learning rate
 is determined by $|| \langle \vec{\gamma}_{\alpha}^s \rangle ||$ (eq. \ref{eq:gamma}), which is in itself determined by both the norms and the number of samples in the effective dataset. Unlocking thus happens at different times for the different groups of neurons in the network. In practice, norm-based pruning at a fixed time early in training may therefore prematurely discard "slow" groups that have yet to enter their unlocking phase.

 \section{Experiments}

 \textbf{General Settings} We created binary classification datasets from CIFAR-10 images: class 1 = ‘airplane' + ‘cat', class 2 = ‘car' + ‘bird', grayscaled and flattened to $1024 \times1$ input vectors. We used batch size 4, learning rate  $\lambda=0.001$ and SGD with momentum 0.9. Our default set of experiments consisted of networks with  d=300 (number of neurons), $std = 0.00001, 0.001 $ (Gaussian initialization) and $N=200, 800, 4000$ (number of samples, balanced). Code is available at [will be added upon acceptance].


\textbf{Computing the alignment} In our theoretical framework, the alignment is given by $\delta_{\alpha,i} = |\theta_{\vec{\gamma}^{*,init}_{c(\alpha),k},i } - \theta_{\w,i}|$, computed based on knowledge of the stable cell. In practice, we cannot straightforwardly predict the stable cell for each neuron at initialization. We compute instead the angular distance to the target vector of the \emph{current} cell   $\overline{\delta_{\alpha}}(t)  = |\theta_{\overline{\langle \vec{\gamma}_{\alpha}^s \rangle} (t)} - \theta_{\w}| $. Here $\theta_{\overline{\langle \vec{\gamma}_{\alpha}^s \rangle} (t)}$ is computed based on the angle $\ww$ will converge to estimated from the current dataset (see theorem in \ref{app:target_outgoing}). In this way, the results also illustrate to what degree the final norms can be predicted from computations at the initialization and very early in training, without information about later timepoints.

\textbf{Reduced effective capacity due to mutual alignment} In a first set of experiments, we consider datasets at different levels of complexity through increasing the dataset size, and we study how the effective capacity is adapted through the existence of highly aligned clusters of neurons that can be merged. In fig. \ref{fig:clustering}(a) we show the number of clusters obtained at cos\_sim threshold 0.95 for models of different sizes $d$ at the end of training, in function of the dataset size N (std=0.00001, average over 5 runs). A higher number of clusters means a higher effective capacity. As expected, the different models use more and more effective capacity as the complexity of the dataset increases, until they saturate their capacity in terms of clusters ($d=n_{clusters}$). Note that these results depend on the selected cluster threshold. In fig. \ref{fig:clustering}(b) we show the evolution of the clustering over 3 different timepoints: initalization, before the loss drops, and at the end of training, for the default set of experiments, using cos\_sim threshold = 0.95. This shows that increasing the initialization scale, thus reducing the initial period of norm-direction decoupling, also increases the number of clusters (i.e., neurons are less closely aligned, thus there are more different clusters). This is as expected, as the period over which neurons can converge in direction before the nature of the dynamics change is reduced. We also find that clusters based on a threshold cos\_sim $\geq0.95$ can be merged without loss in performance (see appendix \ref{app:exp_res}).

\begin{figure}[h!]
    \centering
    \includegraphics[width=\linewidth]{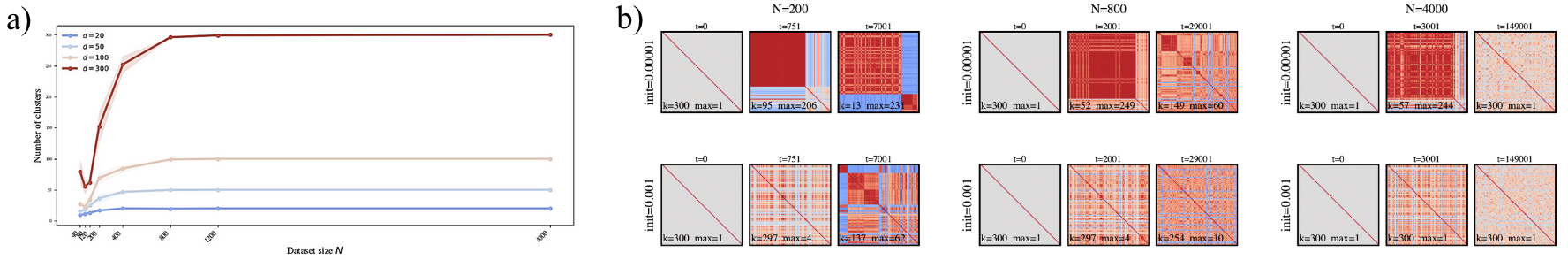}
    \caption{Results of cluster detection. (a) number of clusters in terms of dataset size for models with different number of neurons,  std=0.00001, average over 5 runs. (b) clustering results visualized at different points in time for default set of experiments  (upper row small init, lower row medium init).    }
    \label{fig:clustering}
\end{figure}

\textbf{Exponential growth in norm in function of alignment } In fig. \ref{fig:norm_alignment}, we show the norm $\alpha(t)$ in function of $\overline{\delta_{\alpha}}(t)$ computed at t=0 and a timepoint early in training for all neurons in the network (in fig. \ref{fig:main_figure}, we visualize only the largest cluster) (see appendix fig. \ref{app:loss_curves} for the loss curves). The colors correspond to the \emph{final} norms $a(t=t_{final})$, with darker color a higher value. The distributions at initialization are already predictive of the final norm to a large degree. However, after just a few gradient updates, the outgoing weights have aligned and the dynamics more closely follow the simplified dynamics. We then clearly see the exponential growth of the norm in function of the alignment for all experiments. Note that groups evolving at different speeds are 'mixed' in these visualizations: some neurons of slower groups still need to go through the process of exponential increase at the chosen timepoints.

\begin{figure}[h!]
    \centering
    \includegraphics[width=0.65\linewidth]{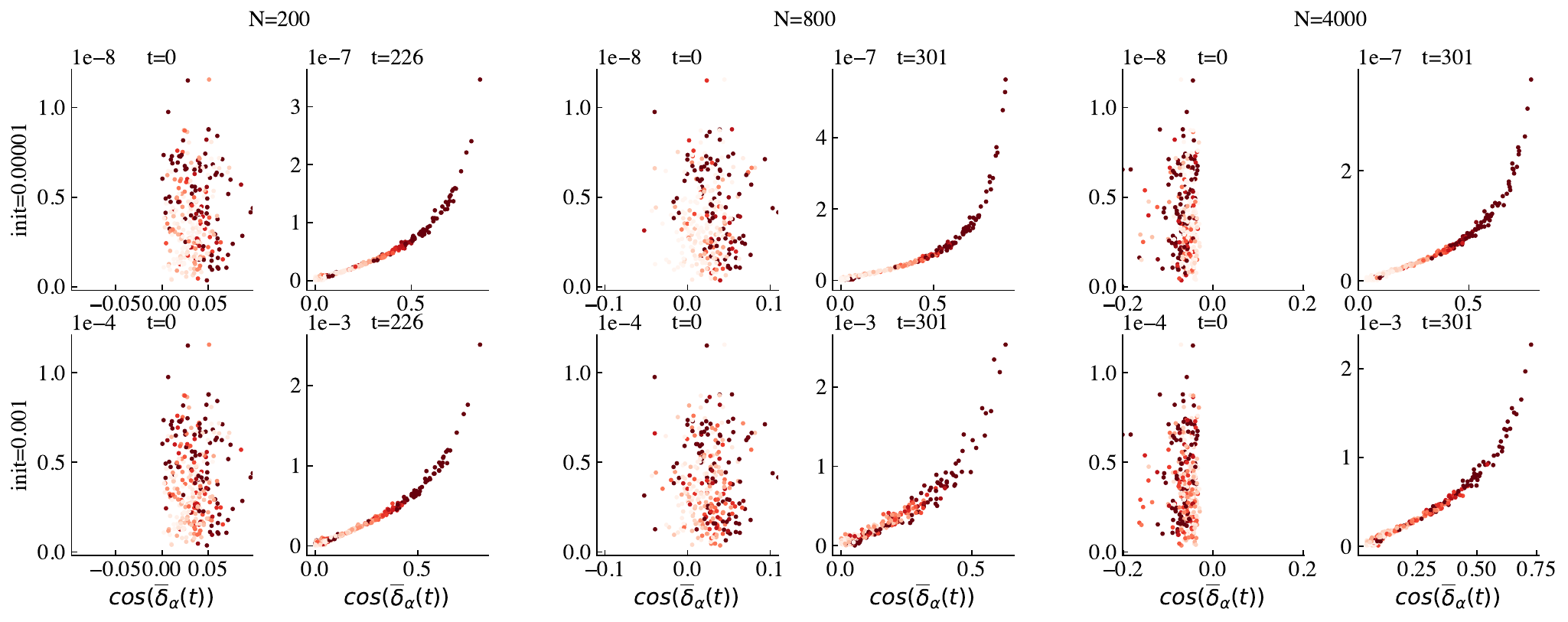}
    \caption{ $a_{\alpha}(t)$ in function of $cos(\overline\delta_\alpha(t)$ computed at t=0 and a timepoint early in training for the default set of experiments. Upper row small init, lower row medium init. Colors correspond to final norm values $a_{\alpha}(t)$. Early in training, we can clearly see the exponential growth of the norm in function of the alignment for all experiments. }
    \label{fig:norm_alignment}
    \vspace*{-10pt}
\end{figure}

\textbf{Comparison to Lottery Tickets} We used the IMP algorithm (\cite{frankle_lottery_2019}) to prune individual weights in the networks to different desired sparsity levels (20\%, 50\%, 70\%,) in 3 rounds.  In fig. \ref{fig:LTH}, we show the resulting number of parameters pruned per neuron in function of the alignment of each neuron, computed at timepoints very early in training for the original networks. Although the context of IMP/Lottery tickets is different (i.e., individual weights vs. neurons, multiple rounds of retraining), we can see that computing the alignment very early in training can be used to predict, to a large degree, which neurons will have a lot of individual parameters pruned when using IMP for lower target sparsities. This works even better if we can split the neurons based on their clusters (see fig. \ref{fig:main_figure}), but this requires knowledge of the clusters.

\begin{figure}[h!]
    \centering
    \includegraphics[width=0.7\linewidth]{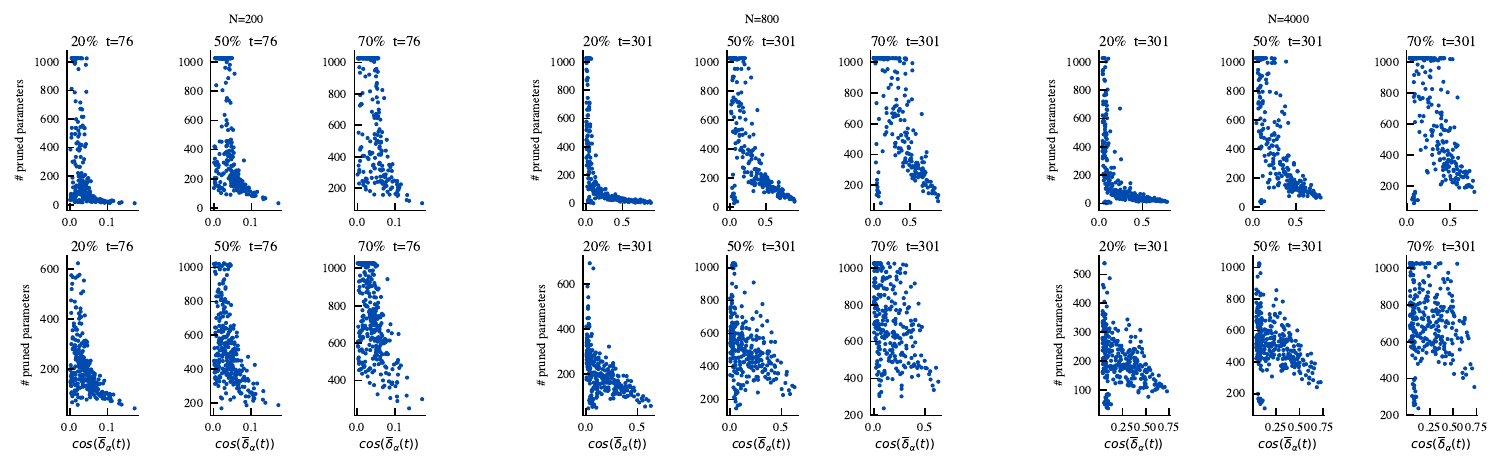}
    \caption{Lottery tickets: number of pruned individual parameters for each neuron in function of $cos(\overline\delta_\alpha(t)$ computed at a timepoint early in training for the default set of experiments, per target sparsity (20\%, 50\%, 70\%). Upper row small init, lower row medium init.We can see that computing the alignment very early in training can be used to predict, to a large degree, which neurons will have a lot of parameters pruned, for lower target sparsities.}
    \label{fig:LTH}
\end{figure}

 \section{Limitations and Conclusion}

In this paper, we have developed a new framework to analyze the early dynamics of gradient descent at the level of individual neurons, and this for single hidden layer ReLU networks for binary classification in the feature learning regime. We identified three dynamical principles, mutual alignment, unlocking and racing, that allowed us to explain how gradient descent adapts a network's theoretical capacity to the task. We specifically showed why some neurons converge in direction to a shared, task-related  target direction; if they align very closely, they can be merged to a single, equivalent neuron. We moreover explained that the beneficial initial conditions of lottery tickets are in fact their favorable initial alignment to task-related target directions. And we identified the lottery selection mechanism as a race between neurons early in training that arises because neurons that align faster to their target direction obtain an exponentially higher norm.\\ 
The limitations of our work are the following: our theory or experiments do not cover deeper networks or networks with more advanced architectures. Our theory does not cover higher initilization scale, but we study this in our experiments. The assumption of aligned outgoing weigths does not hold at initialization but only soon after, such that our predictions of final norm are more accurate after a few gradient updates. Finally, our theory does not cover the dynamics beyond the initial phase of learning. Still, we argue that our results can serve as a theoretical basis to explain a range of earlier observations about lottery tickets: they can be largely predicted early in training \cite{frankle_early_2020, frankle_linear_2020}; the sign of the initialization (hence, the orientation) is more important than the values \cite{zhou2019deconstructing}; and as we found the tickets consist of neurons closely aligned to task-relevant directions, they can generalize across similar datasets \cite{morcos_one_2019}. With more extensions of our framework to advanced architectures, it could serve as a basis to understand capacity large, modern neural networks.

\clearpage

\bibliographystyle{unsrtnat}
\bibliography{references}

\appendix
\onecolumn



\section{Equations of gradient descent}\label{app:eq_grad_descent}

We want to analyze the gradient of the loss with respect to the angles and norms of the ingoing and outgoing weight vectors of each neuron. To achieve this, we start with computing the components of the gradient of the loss $L^s$ with respect to the logits $z^s_k$, for an input sample s. These are given by the simple expression:
\begin{equation}
\frac{\partial L^s}{\partial z_k} = - (y^s_k-\hat{y}^s_k) =  - \Delta y_k^s 
\end{equation}
due to the interaction between the cross-entropy loss and the softmax function. This gradient is then propagated to the outgoing weights of each neuron. We obtain:
\begin{equation}
    (\frac{\partial L^s}{\partial \vec{w}^{(2)}_{\alpha}})^{T} = - g^{s}_\alpha \vec{\Delta y^s} h_{\alpha}^s  
\end{equation}
and
\begin{equation}
    \frac{\partial L^s}{\partial h_\alpha^s} = - g^{s}_\alpha ||\vec{\Delta y^s}|| \; || \vec{w}^{(2)}_{\alpha} || \cos{(\phi_{\Delta y^s} - \theta_{\ww} )}
\end{equation}

Using the chain rule:
\begin{align}
    \frac{\partial L}{\partial ||v||} & =    \frac{\partial L}{\partial v_x }  \frac{\partial v_x }{\partial ||v||} + \frac{\partial L}{\partial v_y }  \frac{\partial v_y }{\partial ||v||} \\ 
    \frac{\partial L}{\partial \theta} & =    \frac{\partial L}{\partial v_x }  \frac{\partial v_x }{\partial \theta} + \frac{\partial L}{\partial v_y }  \frac{\partial v_y }{\partial \theta}  
\end{align}
together with the identities $sin(a-b) = sin(a)cos(b) - cos(a)sin(b)$  and $cos(a-b) = cos(a)cos(b) + sin(a)sin(b) $, and after averaging over the effective dataset, we arrive at:
\begin{align}
      \langle  \frac{\partial L^s}{ \partial \theta_{\ww}} \rangle &=- \langle h_{\alpha}^s  \: ||\Delta  y^s||  \: ||w^{(2)}_{\alpha}|| \, \sin{(\phi_{\Delta y^s} - \theta_{\ww} )} \rangle \\
        \langle \frac{\partial L^s}{ \partial ||w^{(2)}_{\alpha}||} \rangle & = - \langle  h_{\alpha}^s  \: ||\Delta  y^s||   \, \cos{(\phi_{\Delta y^s} - \theta_{\ww} )} \rangle \\
    \end{align}

For the derivatives to $||\w||$ and $\theta_{\w,i}$, we use:
\begin{equation}
    \frac{\partial L^s}{\partial ||\w||} =  \frac{\partial L^s}{\partial h_\alpha^s} \frac{\partial  h_\alpha^s}{\partial  ||\w||}= \frac{\partial L^s}{\partial h_\alpha^s} \frac{\partial (  \w \cdot \vec{x}^s)}{\partial  ||\w||}
\end{equation}
Using the linearity of the derivative, we can move the factors of $ \frac{\partial L^s}{\partial h_\alpha^s}$ to the inner product, and obtain:
\begin{align}
    \langle \frac{\partial L^s}{ \partial (\theta_{\w,i})} \rangle & = - ||\vec{w}^{(2)}_{\alpha}||  \frac{\partial \vec{w}^{(1)}_{\alpha} \cdot \langle\vec{\gamma}_{\alpha}^s \rangle  }{\partial (\theta_{\w,i})}
     \\
       \langle \frac{\partial L^s}{ \partial ||\vec{w}^{(1)}_{\alpha}||} \rangle & = - ||\vec{w}^{(2)}_{\alpha}||  \frac{\partial \vec{w}^{(1)}_{\alpha} \cdot \langle\vec{\gamma}_{\alpha}^s \rangle  }{\partial ||\w|| }
\end{align}
with 
\begin{align}
   & \langle \vec{\gamma}_{\alpha}^s \rangle = \frac{1}{N} \sum^N_s g^{s}_\alpha  \cos(\phi_{\Delta y^{s}} - \theta_{\ww}) \, ||  \Delta \vec{y}^s  || \,  \vec{x}^{s}.
\end{align}
The derivative with respect to the bias is given by:
\begin{equation}
    \langle\frac{\partial L^s}{\partial b_\alpha}\rangle = \langle \frac{\partial L^s}{\partial h_\alpha^s} \frac{\partial h_\alpha^s}{\partial b_\alpha} \rangle = - || \vec{w}^{(2)}_{\alpha} || \; \langle ||\vec{\Delta y^s}|| \;  \cos{(\phi_{\Delta y^s} - \theta_{\ww} )} \rangle
\end{equation}

In total, this yields:
\begin{align}
      & \!\langle \frac{\partial L^s}{ \partial \theta_{\ww}} \rangle =- ||\vec{w}^{(2)} _{\alpha}|| \; \langle h_\alpha^s \;  ||\vec{\Delta y}^{s}|| \;  
      \sin{(\phi_{\vec{\Delta y^{s}}}- \theta_{\ww}}) \rangle  \label{app_eq:grad_phi_alpha} \\ 
     & \! \langle \frac{\partial L^s}{ \partial ||\vec{w}^{(2)}_{\alpha}||} \rangle =  - \langle h_\alpha^s \;  ||\vec{\Delta y}^{s}|| \;  \cos{(\phi_{\vec{\Delta y^{s}}}- \theta_{\ww}}) \rangle \label{app_eq:grad_norm_ww_alpha} \\ 
      & \! \langle \frac{\partial L^s}{ \partial (\theta_{\w,i})} \rangle \! = \! - \!||\vec{w}^{(2)}_{\alpha}|| \; ||\vec{w}^{(1)}_{\alpha} || \; ||\langle\vec{\gamma}_{\alpha}^s \rangle  || \; \sin{(\theta_{\langle\vec{\gamma}_{\alpha}^s \rangle,i } - \theta_{\w,i})} \label{app_eq:grad_theta_alpha} 
     \\ 
      & \! \langle \frac{\partial L^s}{ \partial ||\vec{w}^{(1)}_{\alpha}||} \rangle  = - ||\vec{w}^{(2)}_{\alpha}|| \; ||\langle\vec{\gamma}_{\alpha}^s \rangle  || \;\cos{(\theta_{\langle\vec{\gamma}_{\alpha}^s \rangle,i } - \theta_{\w,i})}\\
       & \langle\frac{\partial L^s}{\partial b_\alpha}\rangle = - || \vec{w}^{(2)}_{\alpha} || \; \langle ||\vec{\Delta y^s}|| \;  \cos{(\phi_{\Delta y^s} - \theta_{\ww} )} 
       \rangle 
      \label{app_eq:grad_b}
    \end{align}

\section{Dynamics of the Outgoing Weights $\ww$}\label{app:ww}

\begin{figure}
    \centering
    \includegraphics[width=\linewidth]{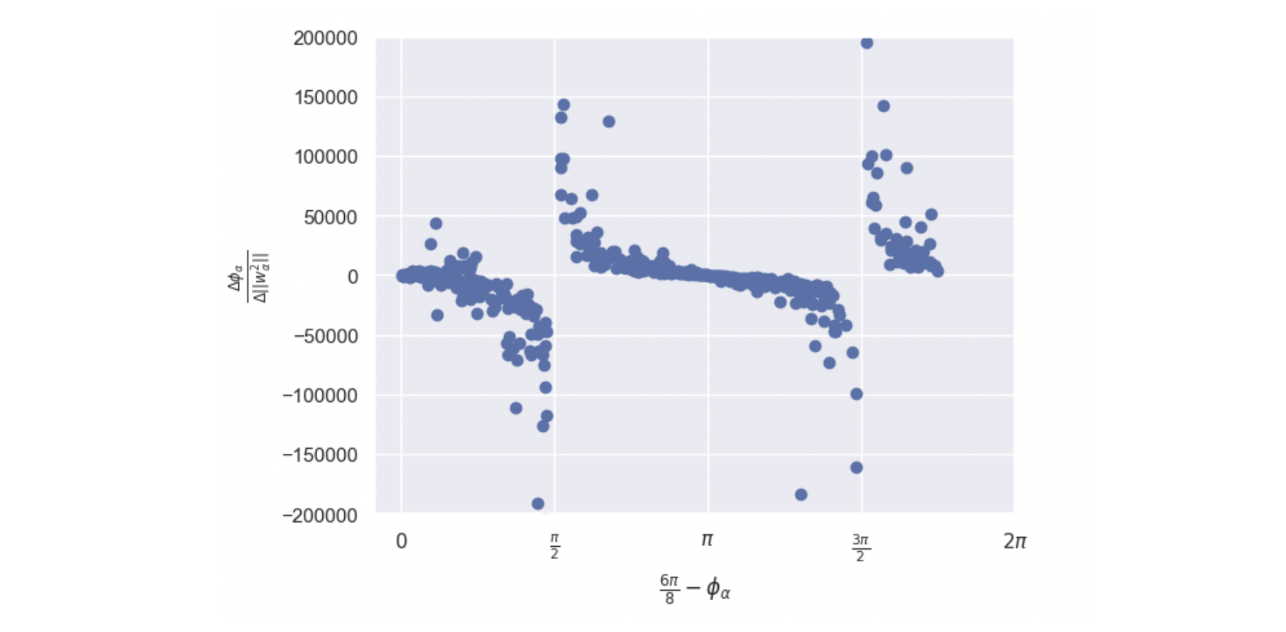}
    \caption{Ratio $  \frac{\partial L^s}{ \partial (\theta_{\ww})} /   \frac{\partial L^s}{ \partial ||\vec{w}^{(2)}_{\alpha}||} $ in function of $\frac{3\pi}{4}- \phi_{\alpha}$ at initalization, for an experiment with 1000 neurons to better visualize the $tan$ function. Note the orders of magnitude on the y-axis. }
    \label{app:decoupling}
\end{figure}

\subsection{The direction of the gradient vectors is fixed} \label{app:fixed_target}
As the predicted output $\hat{y}^s$ is normalized through the softmax, we can write it as $\hat{y}^s = [p, 1-p^s]^T$, for some $p^s \in [0,1]$ depending on the timestep t. We thus obtain
\begin{align}
   \Delta  y^s   = \begin{bmatrix}
    1 \\
    0 \\
\end{bmatrix} -  \begin{bmatrix}
    p \\
    1-p \\
\end{bmatrix} = \begin{bmatrix}
    1-p \\
    -(1-p) \\
\end{bmatrix} or \; 
   \Delta  y^s   = \begin{bmatrix}
    0 \\
    1 \\
\end{bmatrix} -  \begin{bmatrix}
    p \\
    1-p \\
\end{bmatrix} = \begin{bmatrix}
    -p \\
    p \\
\end{bmatrix},
\end{align}

This means that while the magnitude of the vectors $ \vec{\Delta  y}^s $ can change during training, their directions stay constant. The angles $\theta_{\vec{\Delta y}^s}$ are thus fixed at $\theta_{\vec{\Delta y}^s}=\frac{7 \pi}{4}$ for samples of the first class, and at $\theta_{\vec{\Delta y}^s}=\frac{3 \pi}{4}$ for samples of the second class.   \\

\subsection{Target directions}

\begin{theorem} \label{app:target_outgoing}
The target direction for the outgoing weight vector $\theta_{\ww}^*$  is given by: 
\begin{align}
& \theta_{\ww}^* =  \theta_{\ww}^{*1} \equiv  \frac{7 \pi}{4} + n 2\pi \; \text{when} \; \w  \vec{\xi}_1 >  \w \vec{\xi}_2 \\
& \theta_{\ww}^* =  \theta_{\ww}^{*2} \equiv \frac{3 \pi}{4} + n 2\pi \; \text{when} \;  \w  \vec{\xi}_1 <  \w \vec{\xi}_2 \\
&  \vec{\xi}_1 \equiv N^{e}_1 \langle ||\Delta  y^s|| \, \vec{x}^s  \rangle_{C_1}, \; \vec{\xi}_2  \equiv N^{e}_2 \langle ||\Delta  y^s|| \, \vec{x}^s  \rangle_{C_2}.  
\end{align}
\end{theorem}

\textbf{Proof: } We can determine the final angles  $\theta^*_{\ww} $ the angles $\theta_{\ww} $ will converge to from computing fixed points. This means we determine the $\theta^*_{\ww} $ for which $ \frac{\partial \theta_{\ww}^*}{\partial t}  = - \lambda \langle \frac{\partial L^s}{ \partial \theta_{\ww}^*} \rangle = 0$. We repeat the needed equation here for convenience:
\begin{equation}
      -\lambda  \langle \frac{\partial L^s}{ \partial \theta_{\ww}} \rangle = \lambda ||\vec{w}^{(2)} _{\alpha}|| \; \langle h_\alpha^s \;  ||\vec{\Delta y}^{s}|| \;  
      \sin{(\phi_{\vec{\Delta y^{s}}}- \theta_{\ww}}) \rangle 
\end{equation}

Let $s \in C_1$ be the indices for samples of the first class, and $s' \in C_2$ be the indices for the second class. If we write the sums over the effective samples per class explicitly, we obtain that the following must hold:

\begin{align}
    & \sum_{s \in C_1 }   g^{s}_\alpha h_{\alpha}^s  ||\Delta  y^s||  \sin{(\frac{7 \pi}{4} - \theta_{\ww}^* )}  +  \sum_{s' \in C_2 }  g^{s'}_\alpha h_{\alpha}^{s'} ||\Delta  y^{s'}||   \sin{( \frac{3 \pi}{4} - \theta_{\ww}^* )} = 0 \\
         \iff  & \big[ - \sum_{s \in C_1 }   g^{s}_\alpha h_{\alpha}^s  ||\Delta  y^s||    +  \sum_{s' \in C_2 }  g^{s'}_\alpha h_{\alpha}^{s'}   ||\Delta  y^{s'}||    \big]  \sin{(\frac{3 \pi}{4} - \theta_{\ww}^* )}= 0.
\end{align}
The first factor is very unlikely to be exactly zero in practice, so we arrive at:
\begin{equation}
    \theta_{\ww}^* = \frac{3 \pi}{4}  + n \pi.
\end{equation}
Which of these possible fixed points, given by the sets $\frac{3 \pi}{4}  + 2n \pi$ and $\frac{7 \pi} {4}  + 2n \pi$, will the angle converge to? To answer this question, we can determine their stability through the second derivative. We have to determine whether the expression:
\begin{align}
     & - \Big( \sum_{s \in C1 }   g^{s}_\alpha h_{\alpha}^s  ||\Delta  y^s||  \cos{(\frac{7 \pi}{4} - \theta_{\ww}^* )}  + \sum_{s' \in C_2 }  g^{s'}_\alpha h_{\alpha}^{s'} ||\Delta  y^{s'}||   \cos{( \frac{3 \pi}{4} - \theta_{\ww}^* )} \Big)  \\
\end{align}
amounts to a positive (unstable fixed point) or negative value (stable fixed point). Let's start with the case $\theta_{\ww}^* = \frac{7 \pi}{ 3} + n \pi $. We find:
\begin{align}
     & - \Big( \sum_{s \in C1 }   g^{s}_\alpha h_{\alpha}^s  ||\Delta  y^s||  - \sum_{s' \in C_2 }  g^{s'}_\alpha h_{\alpha}^{s'} ||\Delta  y^{s'}||  \Big) < 0 \\
    & \iff \Big( \sum_{s \in C1 }   g^{s}_\alpha h_{\alpha}^s  ||\Delta  y^s||  - \sum_{s' \in C_2 }  g^{s'}_\alpha h_{\alpha}^{s'} ||\Delta  y^{s'}||  \Big) > 0\\
    & \iff \sum_{s \in C1 }   g^{s}_\alpha h_{\alpha}^s  ||\Delta  y^s|| > \sum_{s' \in C_2 }  g^{s'}_\alpha h_{\alpha}^{s'} ||\Delta  y^{s'}||
\end{align}
In other words, when $\sum_{s \in C1 }   g^{s}_\alpha h_{\alpha}^s  ||\Delta  y^s|| > \sum_{s' \in C_2 }  g^{s'}_\alpha h_{\alpha}^{s'} ||\Delta  y^{s'}||$, $\theta_{\ww}^* = \frac{7 \pi}{ 4} + n \pi $ is the stable fixed point. When the reverse holds, $\theta_{\ww}^* = \frac{3 \pi}{ 4} + n \pi $ is the stable fixed point.

\section{Dynamics of the incoming weights $\w$} 
\label{app:fixed_points}

For convenience, we here repeat the assumptions we make:
\begin{definition}\textbf{Simplified dynamics}  For a neuron $\alpha$, we make the following assumptions whenever we assume 'simplified dynamics': $b_\alpha =0$, $||\w||$ constant and small, $||\ww||$ constant and small, and $\theta_{\ww} = \frac{7\pi}{4} + 2\pi$ ("aligned to predict class 1") or  $\theta_{\ww} = \frac{3\pi}{4} + 2\pi$ ("aligned to predict class 2").
\end{definition}

\subsection{Lemma convergence in a stable cell} \label{app:proof_lemma_conv}

\begin{lemma} Consider $\w \in \mathcal{C}_k$ and simplified dynamics. As long as $\w \in \mathcal{C}_k$, $ \langle\vec{\gamma}_{\alpha}^s \rangle (t)$  is  a constant vector given by $\vec{\gamma}^{*,init}_{c(\alpha),k}$ and $\theta_{\w,i}(t) \rightarrow \theta_{\vec{\gamma}^{*,init}_{c(\alpha),k},i } $. 
\end{lemma}

For convenience, we repeat:
\begin{equation}
    \vec{\gamma}_{\alpha} = \sum_{s} \vec{g}_{\alpha}^{s} \cos(\theta_{\Delta y^{s}} - \theta_{\ww}) \, ||  \Delta \vec{y}^s  || \,  \vec{x}^{s} \label{eq_app:gamma}
\end{equation}

\textbf{Proof that $ \langle\vec{\gamma}_{\alpha}^s \rangle (t)$ is constant} We consider the case $\theta_{\ww} = \frac{7\pi}{4} + 2\pi$, the other case is similar. From the derivations in \ref{app:fixed_target}, we have that $\theta_{\vec{\Delta y}^{s}} = \frac{7\pi}{4} + 2\pi $ for  (effective) samples of class 1, and  $\theta_{\vec{\Delta y}^{s}} = \frac{3\pi}{4} + 2\pi $ for (effective) samples of class 2. While $\w \in \mathcal{C}_k$, $\vec{g}_{\alpha} = \vec{g}_k$, i.e., the gating  vector does not change (by definition). We thus obtain:
\begin{equation}
    \vec{\gamma}_{\alpha} = \sum_{s'} \vec{g}_{k}^{s'}  ||  \Delta \vec{y}^s  || \,  \vec{x}^{s'} - \sum_{s''} \vec{g}_{k}^{s''}  ||  \Delta \vec{y}^s  || \,  \vec{x}^{s''}
\end{equation}
where we denoted with $s'$ effective samples of class 1 and with $s''$ effective samples of class 2. We further have that under simplified dynamics  $\vec{\hat{y}}^s \approx [0.5 \;0.5] $, as little signal reaches the output nodes for very small weight values. Therefore we can consider $||\Delta  y^s|| =\frac{\sqrt{2}}{2}$ for all samples during the simplified dynamics. Thus  $ \vec{\gamma}_{\alpha}(t)$ becomes a constant vector, in this case given by $\vec{\gamma}^{*,simpl.}_{1,k} = \frac{\sqrt{2}}{2} ( N^{e}_1 \langle  \, \vec{x}^s  \rangle_{C_1} - N^{e}_2 \langle \, \vec{x}^s  \rangle_{C_2})$.

\textbf{Proof that $\theta_{\w,i}(t) \rightarrow \theta_{\vec{\gamma}^{*,init}_{c(\alpha),k},i } $. } In the case of simplified dynamics and  $|| \vec{\gamma}^{*,simpl.}_{c(\alpha),k} ||$ constant,  eq. \ref{eq:grad_theta_alpha} reduces to:
$\langle \frac{\partial L^s}{ \partial (\theta_{\w,i})} \rangle \! =  - B_{\alpha}  \; \sin{(\theta_{\vec{\gamma}^{*,simpl.}_{\alpha),k},i} - \theta_{\w,i})}  
$ with $B_\alpha$ a constant. First we compute the optima through $\frac{\partial \theta_{\w,i}^*}{\partial t} = 0$:

\begin{align}
& \frac{\partial \theta_{\w,i}^*}{\partial t}  = - \lambda \langle \frac{\partial L^s}{ \partial \theta_{\w,i}^*} \rangle = 0 \\
 & \iff \sin{(\theta_{\vec{\gamma}^{*,simpl.}_{c(\alpha),k} } - \theta^*_{\w,i})} = 0 \\
  & \iff  \theta^*_{\w,i} = \theta_{\vec{\gamma}^{*,simpl.}_{c(\alpha),k}}  + n \pi.
\end{align}

The sign of $  \frac{\partial^2 \theta_{\w,i}^*}{\partial t^2}  $ is determined by the sign of 
\begin{align}
 \cos{(\theta_{\vec{\gamma}^{*,simpl.}_{c(\alpha),k},i } - \theta_{\w,i})} \frac{\partial(\theta_{\vec{\gamma}^{*,simpl.}_{c(\alpha),k},i }- \theta_{\w,i})}{\partial \theta_{\w,i} }
\end{align}

In the regime where $||\Delta  y^s|| = \frac{\sqrt{2}}{2}$ for all samples, $\theta_{\vec{\gamma}^{*,simpl.}_{c(\alpha),k}}$ does not depend on $\theta_{\w,i}$, and we obtain: 
\begin{align}
    \text{a negative sign when  }\theta_{\w,i} = \theta_{\vec{\gamma}^{*,simpl.}_{c(\alpha),k}} + n\pi  \text{ with n even }\\
    \text{a positive sign when  }\theta_{\w,i} = \theta_{\vec{\gamma}^{*,simpl.}_{c(\alpha),k}} + n\pi  \text{ with n odd }
\end{align}
this means $\theta_{\w,i}^* = \theta_{\vec{\gamma}^{*,simpl.}_{c(\alpha),k}} + 2 n \pi$ is the stable fixed point (at least during the initial regime). 

\subsection{Next Order Approximation}

\subsubsection{Approximations to the Softmax }

In general, we have for the output nodes:

\begin{align}
& \hat{y}_0 = \frac{e^{z_0^s}}{e^{z_0^s} + e^{z_1^s}} \\
&\hat{y}_1 = \frac{e^{z_1^s}}{e^{z_0^s} + e^{z_1^s}} 
\end{align}

When $\ww$ is aligned ($\theta_{\ww} = \frac{7\pi}{4} + 2\pi$ ("aligned to predict class 1") or  $\theta_{\ww} = \frac{3\pi}{4} + 2\pi$ ("aligned to predict class 2")), we have either $\ww =  k_\alpha [\begin{smallmatrix} 1 \\ -1 \\\end{smallmatrix}]$ or $\ww =  k_\alpha [\begin{smallmatrix} -1 \\ 1 \\\end{smallmatrix}]$ for a positive scalar $k_\alpha$. If all outgoing weights are aligned, we obtain:
\begin{align}
   & z_0^s = \sum_{\alpha // class 1} k_\alpha h_{\alpha} -  \sum_{\alpha' // class 2} k_{\alpha'} h_{\alpha'} \\
    & z_1^s = - \sum_{\alpha // class 1} k_\alpha h_{\alpha} +  \sum_{\alpha' // class 2} k_{\alpha'} h_{\alpha'} \\
\end{align}
where $\alpha // class 1$ means a sum over the neurons for which $\theta_{\ww} = \frac{7\pi}{4} + 2\pi$.
We thus have, when all outgoing weights are aligned, that $z_1^s = -z_0^s$. The output can thus be described by keeping track of the first output node, and we define $\hat{p}^s = \hat{y}_0^s$. Then we can rewrite $\hat{p}^s$ as:
\begin{align}
    \hat{p}^s = \frac{e^{z_0^s}}{e^{z_0^s}} \frac{e^{z_0^s}}{e^{z_0^s} + e^{z_1^s}} = \frac{1}{1+e^{-2z_0^s}} = \text{sigmoid}(2z_0^s)
\end{align}

We can expand the sigmoid function around the origin:
\begin{equation}
    \text{sigmoid}(x) = \frac{1}{2} + \frac{1}{4}x - \frac{1}{48}x^3 + \cdots
\end{equation}
Thus, a zeroth order approximation to $\hat{p}^s$ is given by
\begin{equation}
    \hat{p}^s \approx \frac{1}{2} 
    \label{eq_app:zeroth_order}
\end{equation}
and a first order approximation is given by:
\begin{equation}
    \hat{p}^s \approx \frac{1}{2} + \frac{1}{2}z_0^s 
    \label{eq_app:first_order}
\end{equation}
These approximations are valid in the early phase of learning when starting from small initial conditions, as in this case $z_0^s$ is very small.

\subsection{Next order approximations to the target directions of $\w$}\label{app:apporx_target_directions}

We here do a similar derivation as the one outlined in \cite{refinetti_neural_2023} for a single perceptron, but now for neuron within a larger network, where we assumed the outgoing weights are aligned. We consider the averaged derivative with respect to the weight vector $\w$ 

\begin{equation}
     \langle \frac{\partial L^s}{ \partial \vec{w}_{\alpha}^{(1)}} \rangle   = -  ||w^{(2)}_{\alpha}||  \vec{\gamma}_{\alpha}
\end{equation}

with $\gamma_{\alpha}$:

\begin{equation}
    \vec{\gamma}_{\alpha} = \sum_{s} g_{\alpha}^{s} \cos(\theta_{\Delta y^{s}} - \theta_{\ww}) \, ||  \Delta \vec{y}^s  || \,  \vec{x}^{s} \label{eq_app:gamma}
\end{equation}

The steady state for this equation is given by (see similar derivation in \cite{refinetti_neural_2023})
\begin{equation}
     \vec{\gamma}_{\alpha} = 0
\end{equation}

and we can use then use the first order approximation eq. \ref{eq_app:first_order}: 
\begin{equation}
    ||  \Delta \vec{y}^{s'}  || = \frac{\sqrt{2}}{2} (1 - z_0^{s'})
\end{equation}
and 
\begin{equation}
    ||  \Delta \vec{y}^{s''}  || = \frac{\sqrt{2}}{2} (1 + z_0^{s''})
\end{equation}
and obtain
\begin{equation}
    z_0^s = \sum_{\beta} g_{\beta}^{s}  \frac{1}{\sqrt{2}}|| \vec{w}_{\beta}^{(2)}|| \cos(\theta_{\Delta y^{s}} - \theta_{\beta})  \vec{w}_{\beta}^{(1)} \vec{x}^s.
\end{equation}

Within the same stable cell, neuron have the same target direction for their incoming weights: 

\begin{equation}
\vec{w}^{(1)*}_{\beta} = k_{\beta} \vec{w}^{(1)*}_{\alpha} 
\end{equation}

with $k_{\beta}$ a positive scalar.

Assume the neuron $\alpha$ is aligned to class 1, $\theta_{\ww} = \theta_{\ww}{*1} $. Define:
\begin{equation}
\vec{m}^{//C_1} = \frac{1}{2} \left[ \sum_{s'} \vec{x}^{s'} - \sum_{s''}  \vec{x}^{s''}  \right]
\end{equation}
and 
\begin{align}
Q^{//C_1} & = - \frac{1}{2} \sum_{\beta} k_{\beta} \left[  \sum_{s' \in C_1} \cos(\theta_{\Delta y^{s'}} - \theta_{\beta}) \vec{x}^{s'} (\vec{x}^{s'})^T  -  \sum_{s'' \in C_2}  \cos(\theta_{\Delta y^{s''}} - \theta_{\beta})  \vec{x}^{s''} (\vec{x}^{s''})^T \right] \\
 & = \sum_{\beta //C_1} k_{\beta} \left[  \sum_{s' \in C_1}  \vec{x}^{s'} (\vec{x}^{s'})^T  +  \sum_{s'' \in C_2}    \vec{x}^{s''} (\vec{x}^{s''})^T \right] -  \sum_{\beta//C_1} k_{\beta} \left[  \sum_{s' \in C_1}  \vec{x}^{s'} (\vec{x}^{s'})^T  +  \sum_{s'' \in C_2}    \vec{x}^{s''} (\vec{x}^{s''})^T \right] \\
 &= (\sum_{\beta //C_1} k_{\beta} - \sum_{\beta//C_1} k_{\beta}) \left[  \sum_{s' \in C_1}  \vec{x}^{s'} (\vec{x}^{s'})^T  +  \sum_{s'' \in C_2}    \vec{x}^{s''} (\vec{x}^{s''})^T \right] 
\end{align}

we then obtain:

\begin{align}
    &   \vec{\gamma}_{\alpha} =  \vec{m}_0^{//C_1}  + Q^{//C_1} \vec{w}^{(1)*}_{\alpha} = 0 
    & \iff \vec{w}^{(1)*}_{\alpha} = -   \vec{m}^{//C_1}(Q^{//C_1})^{-1} 
\end{align}

The zeroth order approximation yields $\vec{m}_0^{//C_1}$, and the second order approximation tells us this vector needs to be rotated according to the second order moments of the datasets as given by $Q^{//C_1}$. Higher order approximations will involve higher order moments of the effective datasets.

\section{Estimating the Number of Stable Cells}\label{app:stable_cells}

\begin{figure}[h!]
    \centering
    \includegraphics[width=0.5\linewidth]{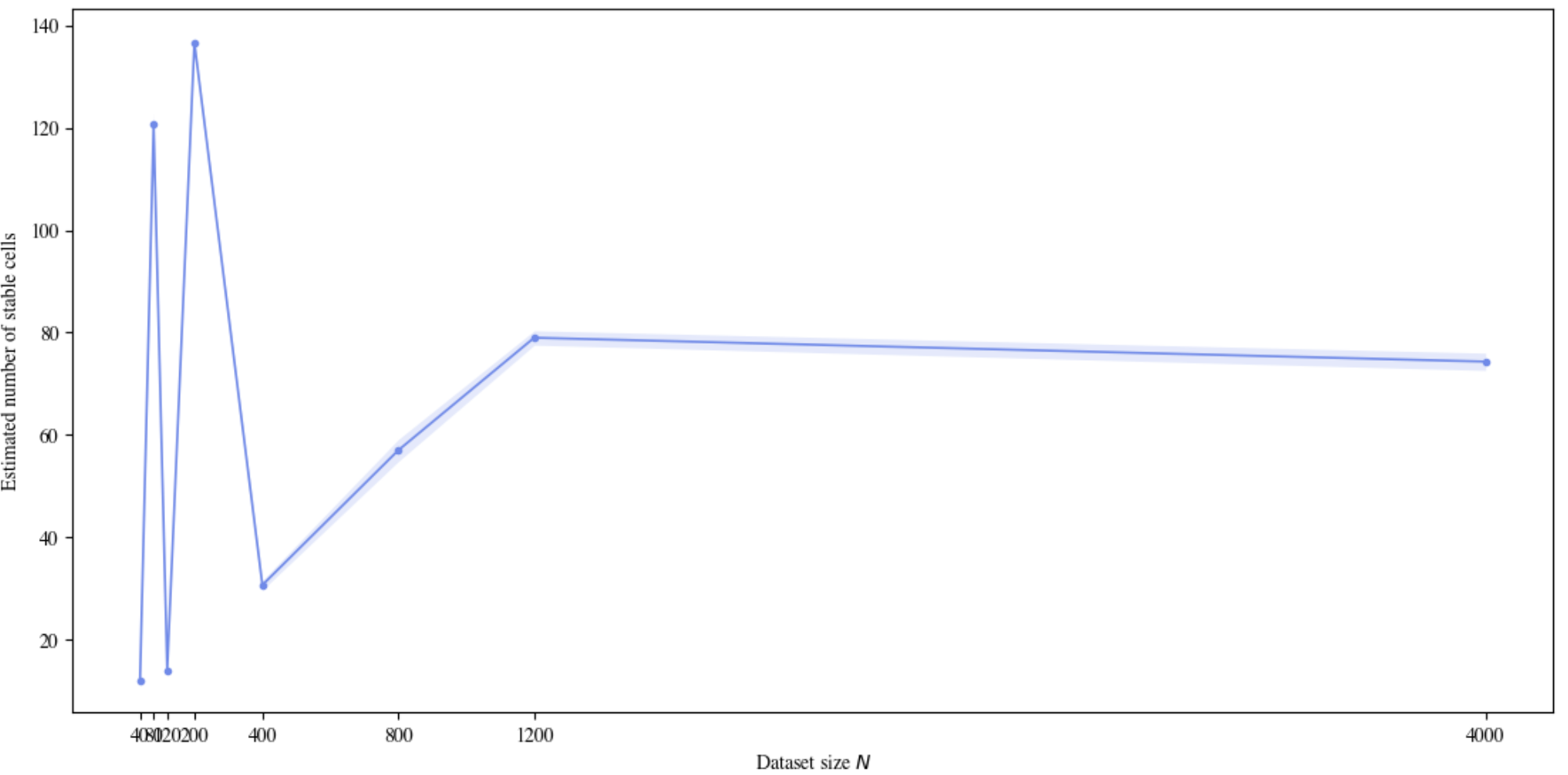}
    \caption{Detected number of stable cells for the binary classification CIFAR-10 datasets of different sizes. Averaged over 3 seeds for the initialization of the M= 100000 probes; $T_max = 200$ and $p=5$.}
    \label{fig:placeholder}
\end{figure}

The algorithm is based on the following observation, formalized below: a vector is the target vector of a stable cell iff, when computing the target vector of this vector, the same vector is returned. We initialize a large number of 'probe' vectors and compute their target vectors, check the criterion, and replace the vectors by their target vectors to continue until we reach convergence. The final number of unique target vectors corresponds to an estimate of the number of stable cells. Note that we run the algorithm twice, because there are two forms of target vectors, based on the two possible alignments of outgoing weights. 

 \paragraph{Setup.}
Let $X \in \mathbb{R}^{N \times d}$ be the dataset with binary labels
$y \in \{0,1\}^N$, and define the signed data matrix with rows
$\tilde{x}_i = (2y_i - 1)\,x_i$.
A weight vector $w \in \mathbb{R}^d$ induces a \emph{gating pattern}
\begin{equation}
    G(w) = \Bigl(\mathbf{1}\bigl[w^\top x_i > 0\bigr]\Bigr)_{i=1}^{N}
    \in \{0,1\}^N,
\end{equation}
recording which inputs lie in the open half-space defined by $w$.
The \emph{target map} $T : \mathbb{R}^d \to \mathcal{S}^{d-1}$ normalises
the sum of signed inputs currently activated by $w$:
\begin{equation}
    T(w) = \frac{m(w)}{\|m(w)\|}, \qquad
    m(w) = \sum_{i=1}^{N} G_i(w)\,\tilde{x}_i = G(w)^\top \tilde{X}.
\end{equation}
A weight vector is a \emph{stable cell} if its gating pattern is a fixed
point of $T$, i.e.\ $G(T(w)) = G(w)$.

\paragraph{Algorithm.}
We sample $M$ probes $w^{(1)}, \ldots, w^{(M)}$ uniformly on
$\mathcal{S}^{d-1}$ and iterate $T$ on all probes simultaneously.
At each step we record $n_t = |\{\text{distinct rows of } G(W)\}|$, the
number of distinct gating patterns in the current population.
Convergence is declared once $n_t$ stops decreasing for $p$ consecutive
iterations (patience criterion); the stable-cell count is then estimated
as $n_{\mathrm{stable}} \leftarrow n_{\mathrm{prev}}$, the last value
recorded before the plateau.
Probes may cycle between two gating patterns rather than settling at a
strict fixed point; using $n_{\mathrm{prev}}$ rather than $n_t$ accounts
for this.
The full procedure is given in Algorithm~\ref{alg:stable-cell}.

\begin{algorithm}[t]
\caption{Stable-Cell Detection}
\label{alg:stable-cell}
\begin{algorithmic}[1]
\Require Dataset $X\!\in\!\mathbb{R}^{N\times d}$, labels
         $y\!\in\!\{0,1\}^N$, probe count $M$,
         iteration cap $T_{\max}$, patience $p$
\Ensure  Stable-cell count $n_{\mathrm{stable}}$
\State Compute $\tilde{X}$ with rows $\tilde{x}_i = (2y_i-1)\,x_i$
\State Sample $W = [w^{(1)},\ldots,w^{(M)}]$ uniformly on $\mathcal{S}^{d-1}$
\State $\texttt{plateau} \leftarrow 0$;\; $n_{\mathrm{prev}} \leftarrow \infty$
\For{$t = 1, \ldots, T_{\max}$}
    \State $G \leftarrow (WX^\top > 0)$ \Comment{gating matrix, shape $M\!\times\!N$}
    \State $W \leftarrow \operatorname{normalise}(G\tilde{X})$ \Comment{apply $T$ to all probes}
    \State $n_t \leftarrow |\{\text{distinct rows of } G(W)\}|$
    \State $\texttt{plateau} \leftarrow \texttt{plateau}+1$ \textbf{if} $n_t \geq n_{\mathrm{prev}}$
           \textbf{else} $0$
    \State $n_{\mathrm{prev}} \leftarrow n_t$
    \If{$\texttt{plateau} \geq p$} \textbf{break} \EndIf
\EndFor
\State \Return $n_{\mathrm{stable}} \leftarrow n_{\mathrm{prev}}$
\end{algorithmic}
\end{algorithm}

\section{Unlocking} \label{app:unlocking}

We start from the equations of the gradient of the norms:
\begin{align}
     & \! \langle \frac{\partial L^s}{ \partial ||\vec{w}^{(2)}_{\alpha}||} \rangle =  - \langle h_\alpha^s \;  ||\vec{\Delta y}^{s}|| \;  \cos{(\theta_{\vec{\Delta y^{s}}}- \theta_{\ww}}) \rangle \\ 
      & \! \langle \frac{\partial L^s}{ \partial ||\vec{w}^{(1)}_{\alpha}||} \rangle  = - ||\vec{w}^{(2)}_{\alpha}|| \; ||\langle\vec{\gamma}_{\alpha}^s \rangle  || \;\cos{(\theta_{\langle\vec{\gamma}_{\alpha}^s \rangle,i } - \theta_{\alpha,i})}
    \end{align}

with $h^s_\alpha = \w \vec{x}^s + b_{\alpha}$, and $\vec{\gamma} _{\alpha}^s$ is given by eq. \ref{eq_app:gamma}. We will use $b_{\alpha} \approx 0$, as it starts from $b_{\alpha} = 0$ at initialization, also only grows significantly when the vector $\w$ is oriented close to its target vector (see eq. \ref{app_eq:grad_b}), but is still orders of magnitude smaller than $||\w||$ at that point (depending on the initialization scale of $\w$). Using 
\begin{align}
   \langle h^s_\alpha  ||\vec{\Delta y}^{s}|| \;  \cos{(\theta_{\vec{\Delta y^{s}}}- \theta_{\ww}}) \rangle = \langle (\w  \cdot \vec{x}^s) ||\vec{\Delta y}^{s}|| \;  \cos{(\theta_{\vec{\Delta y^{s}}}- \theta_{\ww}}) \rangle = ||w^{(1)}_{\alpha}|| \;  || \langle \vec{\gamma}_{\alpha}^s \rangle || cos(\delta_\alpha)
\end{align}
with  $\delta_{\alpha,i}$ be the remaining angular distance, $\delta_{\alpha,i} = \delta_{\alpha,i}(t) = \theta_{\langle\vec{\gamma}_{\alpha}^s \rangle,i } - \theta_{\alpha,i}$. We thus obtain:
\begin{align}
     \frac{\partial L}{ \partial ||  \vec{w}_{\alpha}^{(1)} || }  = -  ||w^{(2)}_{\alpha}|| \;  || \langle \vec{\gamma}_{\alpha}^s \rangle ||  cos(\delta_\alpha)\\
     \frac{\partial L}{ \partial ||  \vec{w}_{\alpha}^{(2)} || }  = -  ||w^{(1)}_{\alpha}|| \;  || \langle \vec{\gamma}_{\alpha}^s \rangle || cos(\delta_\alpha)
\end{align}

We first note, as also discussed in \cite{tarmoun_understanding_2021} in a different context, that we can make use of a conserved quantity:

\begin{equation}
     \frac{\partial ( ||  \vec{w}_{\alpha}^{(2)} ||^2 - ||  \vec{w}_{\alpha}^{(1)} ||^2)}{ \partial t }  = 2  ||\vec{w}_{\alpha}^{(2)}||   \frac{\partial ||\vec{w}_{\alpha}^{(2)}||}{ \partial t} - 2  ||\vec{w}_{\alpha}^{(1)}||   \frac{\partial  || \vec{w}_{\alpha}^{(1)}||}{ \partial t} = 0 
\end{equation}

which yields:

\begin{equation}
     ||  \vec{w}_{\alpha}^{(2)} ||^2 - ||  \vec{w}_{\alpha}^{(1)} ||^2 = c_\alpha
\end{equation}
with $c_\alpha$ a constant. 

Consider the product $a_{\alpha} =  ||  \vec{w}_{\alpha}^{(2)} || \; ||  \vec{w}_{\alpha}^{(1)} ||$. We can then obtain:
\begin{align}
    \big( ||  \vec{w}_{\alpha}^{(2)} ||^2 - ||  \vec{w}_{\alpha}^{(1)} ||^2 \big)^2 & = c_{\alpha}^2 \\
   &  =  ||  \vec{w}_{\alpha}^{(2)} ||^4 + 2 ||  \vec{w}_{\alpha}^{(2)} ||^2 ||  \vec{w}_{\alpha}^{(1)} ||^2 +  ||  \vec{w}_{\alpha}^{(1)} ||^4 \\
   & = ||  \vec{w}_{\alpha}^{(2)} ||^4 - 2 a_{\alpha}^2 +  ||  \vec{w}_{\alpha}^{(1)} ||^4 \\
   \iff c_{\alpha}^2 + 2 a_{\alpha}^2 &=   ||  \vec{w}_{\alpha}^{(2)} ||^4  + || \vec{w}_{\alpha}^{(1)} ||^4 
\end{align}

yielding:
\begin{align}
    \frac{\partial a_\alpha}{ \partial t } & = - \langle \lambda \big(||\vec{w}_{\alpha}^{(1)}||   \frac{\partial ||\vec{w}_{\alpha}^{(2)}||}{ \partial t}  +  ||\vec{w}_{\alpha}^{(2)}||   \frac{\partial  || \vec{w}_{\alpha}^{(1)}||}{ \partial t} \big) \rangle\\
    &= \lambda  || \langle \vec{\gamma}_{\alpha}^s \rangle || \big( ||\vec{w}_{\alpha}^{(1)}||^2 +  ||\vec{w}_{\alpha}^{(2)}||^2  \big)  \\
    & = \lambda  || \langle \vec{\gamma}_{\alpha}^s \rangle || cos(\delta_\alpha) \sqrt{ c_{\alpha}^2 + 4 a_{\alpha}^2 }
\end{align}

As discussed in \cite{tarmoun_understanding_2021}, under the assumption of small initial conditions, we can assume $ c_\alpha \approx 0$. If we furthermore assume $\ww$ is already aligned and we are in the phase of learning where $ ||  \Delta \vec{y}^s  || \approx \frac{\sqrt{2}}{2}$ for all s in the effective dataset, we have that $|| \langle \vec{\gamma}_{\alpha}^s \rangle||$ is a constant (during this phase). Furthermore, in the case of small angular distance, we can assume (due to the direction-norm decoupling) that $\delta_{\alpha,i}(t)$ changes very slowly compared to $a(t)$. This yields: 
\begin{equation}
    a_{\alpha}(t) \sim  e^{ \lambda || \langle \vec{\gamma}_{\alpha}^s \rangle||  cos(\delta_\alpha) t}
\end{equation}
 This represents an exponential growth of $a_{\alpha} =  ||  \vec{w}_{\alpha}^{(2)} || \; ||  \vec{w}_{\alpha}^{(1)} || $ with an exponent determined by $cos(\delta_\alpha) <= 1$. Thus neurons that are closer to their target direction grow exponentially faster in norm.

 \section{Aligned neurons reduce the loss} \label{app_loss_reduction}

 To show that aligned neurons that grow in norm cause the loss to drop, we show that $\frac{\partial \langle ||\Delta  y^{s}||\rangle}{ \partial a_{\alpha}}$ is negative for an aligned neuron $\alpha$. 
 
 Here $a_{\alpha} =  ||  \vec{w}_{\alpha}^{(2)} || \; ||  \vec{w}_{\alpha}^{(1)} ||$. 

We have for $||\Delta  y^{s}|| $ when $s$ is a sample of class 1:
\begin{equation}
   ||\Delta  y^{s}|| = \sqrt{(1-\hat{p^s})^2 + (0-(1-\hat{p^s}))^2} = \sqrt{2} (1 - \hat{p^s})
\end{equation}
and for samples of class 2:
\begin{equation}
   ||\Delta  y^{s}|| = \sqrt{(0-\hat{p^s})^2 + (1-(1-\hat{p^s}))^2} = \sqrt{2}  \hat{p^s}
\end{equation}

Thus, we have, for the average over the effective dataset of neuron $\alpha$:
\begin{align}
     \langle ||\Delta  y^{s}||\rangle & = \frac{N^e_1}{N^e}\langle ||\Delta  y^{s}||\rangle_{C_1} + \frac{N^e_2}{N^e} \langle ||\Delta  y^{s}||\rangle_{C_2}\\
    & = \frac{N^e_1}{N^e} \sqrt{2} \big(1 - \langle \hat{p^s }\rangle_{C_1}\big) + \frac{N^e_2}{N^e} \sqrt{2}  \langle \hat{p^s} \rangle_{C_2}
\end{align}

In the regime we are considering, we have $\hat{p}^s \approx \frac{1}{2} + \frac{1}{2}z_0^s 
$, and 
\begin{align}
   z_0^s  =    \sum_{\beta' // class 1} g_{\beta'}^s w^{(2)}_{\beta',0} \vec{w}^{(1)}_{\beta'} \vec{x}^s +  \sum_{\beta'' // class 2} g_{\beta''}^s w^{(2)}_{\beta'',0} \vec{w}^{(1)}_{\beta''} \vec{x}^s 
\end{align}
Here the index $\beta'$ runs over all neurons aligned to predict class 1 ($\theta_{\beta'} = \theta^{1*}$) and the index $\beta{''}$ runs over all neurons aligned to predict class 2 ($\theta_{\beta'} = \theta^{2*}$). We include the gating vectors of all neurons to indicate if their effective datasets also contains the sample s. I.e., $g_{\beta'}^s = 1$ if the given sample (here an effective sample for neuron $\alpha$) is also an effective sample for neuron ${\beta'}$, and $g_{\beta'}^s = 0$ otherwise. 

Assuming $\alpha // class 1$ (the other case is similar), the contribution of neuron $\alpha$ to $z_0^s$ is given by:
\begin{equation}
    \langle z_0^{s,\alpha} \rangle =  w^{(2)}_{\alpha,0} \vec{w}^{(1)}_{\alpha'} \langle\vec{x}^s\rangle
\end{equation}
and for its contribution to $ \langle ||\Delta  y^{s}||\rangle$, we get:

\begin{align}
     \langle ||\Delta  y^{s}||\rangle ^{\alpha} 
    & = \frac{N^e_1}{N^e} \sqrt{2} \big(\frac{1}{2} - \frac{1}{2}\langle z_0^{s,\alpha}  \rangle_{C_1}\big) + \frac{N^e_2}{N^e} \sqrt{2} (  \frac{1}{2} + \frac{1}{2} \langle z_0^{s,\alpha}  \rangle_{C_2})\\
    & = c + \frac{\sqrt{2}}{2}w^{(2)}_{\alpha,0} \vec{w}^{(1)}_{\alpha'} \big[ - \frac{N^e_1}{N^e}  \langle\vec{x}^s\rangle_{C_1} + \frac{N^e_2}{N^e}  \langle\vec{x}^s\rangle_{C_2} \big]
\end{align}
with c a constant. 
As $\alpha // class 1$, meaning $\theta_{\ww} = \theta^{1*} = \frac{7 \pi}{4}$, we have that $w^{(2)}_{\alpha,0}  = \frac{ || \vec{w}^{(2)}_{\alpha} ||}{\sqrt{2}}$. Moreover, when aligned and in the early phase of training, we know the direction of  $\vec{w}^{(1)}_{\alpha}$ is close to the direction of $\vec{m^1} = + \frac{N^e_1}{N^e}  \langle\vec{x}^s\rangle_{C_1} - \frac{N^e_2}{N^e}  \langle\vec{x}^s\rangle_{C_2} $, which is the opposite direction to the vector in the equation above. We thus obtain:
\begin{align}
    \frac{\partial \langle ||\Delta  y^{s}||\rangle}{ \partial a_{\alpha}}& = \frac{\partial}{\partial a_{\alpha}} \big[ \frac{1}{2}  || \vec{w}^{(2)}_{\alpha} ||\; || \vec{w}^{(1)}_{\alpha}|| \; || \vec{m}^{1} || \cos{(\pi + \Delta)} \big]\\
    & = \frac{1}{2} || \vec{m}^{1} || \cos{(\pi + \Delta)}
\end{align}
with $\Delta$ the angular distance between $\vec{m^1}$ and $\w$, a small angle. The sign of the derivative is thus negative.

\section{Additional Experimental Results} \label{app:exp_res}

\begin{table}[h!] 
\centering
\caption{Effect of cluster merging on model accuracy for varying thresholds $\tau$.
         $k$: number of clusters after merging; $\Delta\mathrm{acc}$: relative change
         in accuracy (\%). Values show mean $\pm$ std over seeds.}
\label{tab:cluster_merge}
\resizebox{\linewidth}{!}{%
\begin{tabular}{llrrrrrr}
\toprule
 & & \multicolumn{2}{c}{$\tau = 0.70$}
   & \multicolumn{2}{c}{$\tau = 0.95$}
   & \multicolumn{2}{c}{$\tau = 0.99$} \\
\cmidrule(lr){3-4}\cmidrule(lr){5-6}\cmidrule(lr){7-8}
$\sigma_0$ & $N$
  & $k$ & $\Delta\mathrm{acc}$ (\%)
  & $k$ & $\Delta\mathrm{acc}$ (\%)
  & $k$ & $\Delta\mathrm{acc}$ (\%) \\
\midrule
\multirow{3}{*}{$10^{-5}$}
  & 200  & $3.3  \pm 0.5$  & $-5.4  \pm 3.8$ & $15.3  \pm  6.3$ & $-3.3 \pm 4.6$ & $ 93.3 \pm 18.6$ & $+0.7 \pm 0.7$ \\
  & 800  & $4.7  \pm 0.9$  & $-29.2 \pm 1.9$ & $156.0 \pm 11.3$ & $+0.5 \pm 0.7$ & $277.0 \pm  5.1$ & $ 0.0 \pm 0.0$ \\
  & 4000 & $105^*$         & $-15.0^*$        & $300^*$           & $0.0^*$         & $300^*$           & $0.0^*$ \\
\midrule
\multirow{3}{*}{$10^{-3}$}
  & 200  & $21.0 \pm 3.7$  & $-14.2 \pm 0.6$ & $172.0 \pm 25.7$ & $-0.2 \pm 0.2$ & $286.3 \pm  7.4$ & $ 0.0 \pm 0.0$ \\
  & 800  & $8.7  \pm 2.5$  & $-29.5 \pm 3.1$ & $256.7 \pm  5.2$ & $ 0.0 \pm 0.0$ & $298.3 \pm  1.7$ & $ 0.0 \pm 0.0$ \\
  & 4000 & $188^*$         & $-4.7^*$         & $300^*$           & $0.0^*$         & $300^*$           & $0.0^*$ \\
\bottomrule
\multicolumn{8}{l}{\footnotesize $^*$ Single seed; no std available.}
\end{tabular}
}
\end{table}

\begin{figure}
    \centering
    \includegraphics[width=\linewidth]{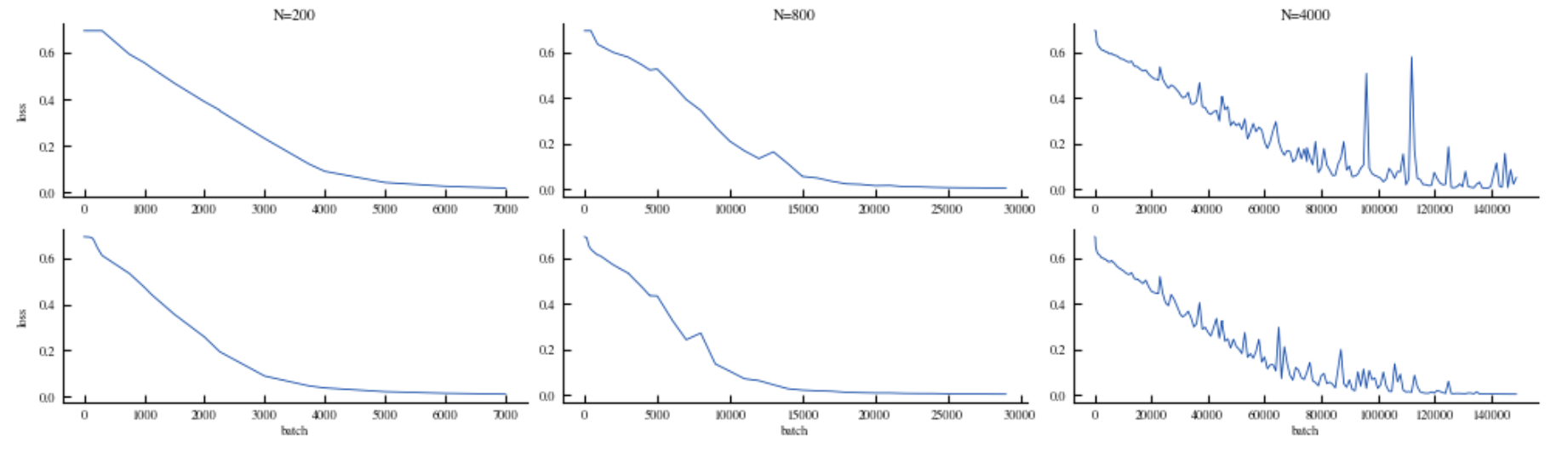}
    \caption{Train loss curves for default set of experiments}
    \label{app:loss_curves}
\end{figure}

\newpage

\end{document}